\documentclass[twoside]{article}

\usepackage[preprint]{aistats2026}
\usepackage[round]{natbib}

\usepackage{mathtools} % amsmath with fixes and additions
\usepackage{booktabs} % commands to create good-looking tables
\usepackage{tikz} % nice language for creating drawings and diagrams
\usepackage{multirow}
\usepackage{graphicx}
\usepackage{subcaption}

\usepackage{amsmath}
\usepackage{amssymb}
\usepackage{mathtools}
\usepackage{amsthm}
\usepackage{bbm}

\usepackage{cancel}

\usepackage[capitalize,noabbrev]{cleveref}

\crefname{defn}{Definition}{Definition}
\crefname{section}{Section}{Section}
\crefname{algorithm}{Algorithm}{Algorithm} 
\crefname{thm}{Theorem}{Theorem}
\crefname{lem}{Lemma}{Lemma}
\crefname{prop}{Proposition}{Proposition}
\crefname{asm}{Asm.}{Asm.}
\crefname{appendix}{Appendix}{Appx.}
\crefname{equation}{Equation}{Equations}
\crefname{figure}{Figure}{Figures}
\crefname{table}{Table}{Tables}
\crefname{cor}{Corollary}{Corollary}

\usepackage{notation}
\usepackage{researchpack}

\begin{document}

% If your paper is accepted and the title of your paper is very long,
% the style will print as headings an error message. Use the following
% command to supply a shorter title of your paper so that it can be
% used as headings.
%
%\runningtitle{I use this title instead because the last one was very long}

% If your paper is accepted and the number of authors is large, the
% style will print as headings an error message. Use the following
% command to supply a shorter version of the author names so that
% they can be used as headings (for example, use only the surnames)
%
%\runningauthor{Surname 1, Surname 2, Surname 3, ...., Surname n}

\twocolumn[

\aistatstitle{Rethinking Probabilistic Circuit Parameter Learning}

\aistatsauthor{Anji Liu\textsuperscript{\rm 1} \And Zilei Shao\textsuperscript{\rm 2} \And  Guy Van den Broeck\textsuperscript{\rm 2}}

\aistatsaddress{\textsuperscript{\rm 1}School of Computing, National University of Singapore\\

\textsuperscript{\rm 2}University of California, Los Angeles } ]

\begin{abstract}
    Probabilistic Circuits (PCs) offer a computationally scalable framework for generative modeling, supporting exact and efficient inference of a wide range of probabilistic queries. While recent advances have significantly improved the expressiveness and scalability of PCs, effectively training their parameters remains a challenge. In particular, a widely used optimization method, full-batch Expectation-Maximization (EM), requires processing the entire dataset before performing a single update, making it ineffective for large datasets. Although empirical extensions to the mini-batch setting, as well as gradient-based mini-batch algorithms, converge faster than full-batch EM, they generally underperform in terms of final likelihood. We investigate this gap by establishing a novel theoretical connection between these practical algorithms and the general EM objective. Our analysis reveals a fundamental issue that existing mini-batch EM and gradient-based methods fail to properly regularize distribution changes, causing each update to effectively ``overfit'' the current mini-batch. Motivated by this insight, we introduce an\textbf{em}one, a new mini-batch EM algorithm for PCs. An\textbf{em}one applies an implicit adaptive learning rate to each parameter, scaled by how much it contributes to the likelihood of the current batch. Across extensive experiments on language, image, and DNA datasets, an\textbf{em}one consistently outperforms existing optimizers in both convergence speed and final performance.
    
    % While empirical extensions to the mini-batch setting have been proposed,  it remains unclear what objective these algorithms are optimizing, making it difficult to assess their theoretical soundness. This paper bridges the gap by establishing a novel connection between the general EM objective and the standard full-batch EM algorithm. Building on this, we derive a theoretically grounded generalization to the mini-batch setting, which leads to a new EM-style algorithm \anji{Algorithm name} \anji{Add intuition (adaptive step size), so we can use much smaller batch sizes - > better convergence and also end performance.}. Extensive empirical evaluations \anji{Summarize empirical results.}
\end{abstract}

\section{Introduction}
\label{sec:intro}

Probabilistic Circuits (PCs) are a class of generative models that represent probability distributions by recursively composing simpler distributions through sum (mixture) and product (factorization) operations \citep{choi2020probabilistic}. The key idea behind PCs is to examine how tractable probabilistic models, such as Hidden Markov Models \citep{rabiner1986introduction}, perform inference (\eg computing marginal probabilities). PCs distill the structure of these models' computation graphs into a compact and general framework, which leads to a unified, computation-oriented perspective on tractable probabilistic modeling.

While significant progress has been made in improving the expressiveness of PCs through architectural innovations \citep{loconte2025sum,liu2021tractable} and system-level advancements \citep{liu2024scaling,peharz2020einsum}, there is still no clear consensus on how to effectively learn their parameters. Full-batch Expectation-Maximization (EM) and its empirical variants remain widely used approaches \citep{zhang2025scaling,liu2023image}. However, full-batch EM requires aggregating information across the entire dataset before each parameter update, making it hard to scale to large datasets or streaming settings. Although mini-batch extensions and gradient-based optimization methods can converge faster, they typically achieve lower final log-likelihood than full-batch EM.

Based on \citet{kunstner2021homeomorphic}, which studies the full-batch EM algorithm for exponential-family latent variable models, we discover that the full-batch EM update of PCs corresponds to optimizing a 1st order Taylor approximation, regularized by a Kullback–Leibler (KL) divergence that penalizes deviation from the current distribution. This yields a novel view of the full-batch EM update for PCs, which has appeared in various forms across different contexts \citep{peharz2015foundations,choi2021group,poon2011sum}.

This perspective naturally suggests a theoretically grounded mini-batch extension: by increasing the weight on the KL term, we can compensate for the reduced information available in a mini-batch compared to the full dataset. The resultant algorithm, an\textbf{em}one (``an EM one''), adaptively applies large learning rates only to the PC parameters that strongly influence the likelihood of the current batch. The other parameters are kept (almost) unchanged to prevent excessive distribution drift, thereby ensuring that updates remain consistent with the data distribution.

The theoretical insights also explain the inferior performance of existing EM- and gradient-based mini-batch algorithms, as they regularize the distribution shift before and after an update using the KL divergence of the local distributions defined at sum nodes\footnote{See \cref{sec:background} for the definition of sum nodes.} and the L2 distance in the parameter space, respectively. As a result, these methods fail to effectively minimize the distribution shift while learning to improve the likelihood given the current batch, causing each update to ``overfit'' to the current samples and impede overall convergence.

An\textbf{em}one admits a closed-form expression, making it efficient and easy to implement. We conduct extensive empirical evaluations on three types of datasets (language, image, and DNA) and four widely used classes of PC architectures. The results demonstrate that an\textbf{em}one consistently and significantly outperforms existing optimizers in both convergence speed and final likelihood.

\section{Background}
\label{sec:background}

% This section introduces a circuit representation of distributions (Sec.~\ref{sec:bg-pc}) and provides the necessary background notations for EM optimization (Sec.~\ref{sec:bg-em}).

\subsection{Distributions as Circuits}
\label{sec:bg-pc}

Probabilistic Circuits (PCs) represent probability distributions with deep and structured computation graphs that consist of sum and product operations \citep{choi2020probabilistic}. They serve as a general framework encompassing tractable probabilistic models, which are designed to support efficient and exact probabilistic inference over complex queries, such as Sum Product Networks \citep{poon2011sum}, cutset networks \citep{rahman2014cutset}, Hidden Markov Models \citep{rabiner1986introduction}, and Probabilistic Generating Circuits \citep{zhang2021probabilistic}. 
% PCs inherit the tractability of these models while introducing a scalable computational framework to encode richer and more complex data distributions. 
% \guy{I would prefer to not call them an extension or richer or introducing scalable framework.
% I think the point remains that PCs is a general name for things like SPNs etc but there is no claim that they are more powerful. it's more interesting to add PGCs as an example of PCs to show that there are types of PCs that are more recent than SPN. In the end it's just an umbrella term, not something we want to offend the TPM community with.}
The syntax and semantics of PCs are as follows:

\begin{defn}[Probabilistic Circuit]
\label{defn:pc}
    A PC $\p$ over variables $\X$ is a directed acyclic computation graph with one single root node $n_{\mathrm{r}}$. Every input node (those without incoming edges) in $\p$ defines an univariate distribution over variable $X \in \X$. Every inner node (those with incoming edges) is either a \emph{product} or a \emph{sum} node, where each product node encodes a factorized distribution over its child distributions and each sum node represents a weighted mixture of its child distributions. Formally, the distribution $\p_{n}$ encoded by a node $n$ can be represented recursively as
        \begin{align}
            \!\!\!\p_{n} (\x) \!:=\!\! \begin{cases}
                f_{n} (\x) & n \text{~is~an~input~node}, \\
                \prod_{c \in \ch(n)} \p_{c} (\x) & n \text{~is~a~product~node},\!\!\! \\
                \sum_{c \in \ch(n)} \theta_{n,c} \!\cdot\! \p_{c} (\x) \!\!\!\!\! & n \text{~is~a~sum~node},
            \end{cases}
            \label{eq:pc-fwd}
        \end{align}
    \noindent where $f_{n}$ is an univariate primitive distribution defined over $X \in \X$ (\eg Gaussian, Categorical), $\ch(n)$ denotes the set of child nodes of $n$, and $\theta_{n,c} \geq 0$ is the parameter corresponds to the edge $(n,c)$ in the PC. Define the \emph{log-parameter} of $(n,c)$ as $\phi_{n,c} := \log \theta_{n,c}$, which will be used interchangeably with $\theta_{n,c}$. We further denote $\params := \{\phi_{n,c}\}_{(n,c)}$ as the set of all sum node parameters in the PC. Without loss of generality, we assume that every path from the root node to an input node alternates between sum and product nodes.\footnote{This can be efficiently enforced since we can directly ``collapse'' consecutive sum nodes or product nodes.}
\end{defn}

To ensure the exact and efficient computation of various probabilistic queries, including marginalization and moment calculations, we must impose structural constraints on the circuit. Specifically, smoothness and decomposability \citep{peharz2015theoretical} are a set of sufficient conditions that ensure tractable computation of marginal and conditional probabilities. This tractability arises because smooth and decomposable circuits represent multilinear functions, which are known to support efficient marginalization \citep{broadrick2024polynomial}. We provide details in \cref{appx:structural-props}.

% \guy{smoothness and decomposability are not used directly, so safe to move to appendix? Or we cite Oliver's work and just say these represent multilinear functions, without getting into the details of how.}

PCs can be viewed as latent variable models with discrete latent spaces \citep{peharz2016latent}. Each sum node can be interpreted as introducing a discrete latent variable $Z$ that selects among its child distributions. Specifically, assigning $Z = i$ corresponds to choosing the $i$-th child of the sum node. By aggregating all such latent variables, the PC can be seen as defining a hierarchical latent variable model.

\subsection{Expectation-Maximization}
\label{sec:bg-em}

Expectation-Maximization (EM) is a well-known algorithm to maximize the log-likelihood given data $\x$ of a distribution defined by a latent variable model. Specifically, the distribution $\p_{\params} (\X)$ with parameters $\params$
% \guy{did we already say that the parameters phi are log-thetas from the previous section? this is important for when we take derivates at the end of this section and what follows} 
is defined as $\sum_{\z} \p_{\params} (\X, \z)$ over latents $\Z$. Our goal is to maximize
    \begin{align}
        \LL (\params) := \log \p_{\params} (\x) = \log \Big ( \sum_{\z} \p_{\params} (\x, \z) \Big ).
        \label{eq:mle-obj}
    \end{align}
EM can effectively maximize the above objective when $\p_{\params} (\X, \Z)$ permits much simpler (or even closed-form) maximum likelihood estimation. It optimizes $\LL (\params)$ by maximizing the following surrogate objective:
    \begin{align}
        Q_{\params} (\params') := \sum_{\z} \p_{\params} (\z \given \x) \cdot \log \p_{\params'} (\x, \z).
        \label{eq:m-step-obj}
    \end{align}
EM updates the current parameters $\params$ by solving for $\params'$ that maximizes $Q_{\params} (\params')$, which is guaranteed to be a lower bound of $\LL (\params')$ since
    \begin{align*}
        Q_{\params} (\params') = \LL (\params') \!+\! \sum_{\z} \p_{\params} (\z \given \x) \cdot \log \p_{\params'} (\z \given \x) \leq \LL (\params').
    \end{align*}

\section{EM for Probabilistic Circuits}
\label{sec:em-for-pcs}

While variants of the EM algorithm have been proposed for training PCs in various contexts \citep{poon2011sum,peharz2015foundations}, their connection to the general EM objective $Q_{\params} (\params')$ (cf. Eq.~(\ref{eq:m-step-obj})) remains implicit. The lack of a unified formulation makes it difficult to fully understand the existing optimization procedures or to extend them to new settings, such as training with mini-batches of data, which is critical for scaling the optimizer to large datasets.

Specifically, there are multiple ways to define mini-batch EM algorithms that all reduce to the same full-batch EM algorithm in the limit. However, it is often unclear what objective these variants are optimizing in the mini-batch case, which complicates the design of new learning algorithms.

In this section, we bridge this gap by deriving EM for PCs explicitly from the general objective. 
% \guy{here, and maybe in the intro, I would emphaize that there are multiple ways to do minibatch EM that all have full batch EM as their extreme case. otherwise the reader might be confused as to whether there even is a problem to be solved.}
In \cref{sec:full-batch-em}, we begin with a derivation for the full-batch case, showing how existing formulations can be recovered and interpreted from this viewpoint. We then extend the derivation to the mini-batch setting in \cref{sec:mini-batch-em}, leading to a principled and theoretically-grounded mini-batch EM algorithm for PCs.

\subsection{Revisiting Full-Batch EM}
\label{sec:full-batch-em}

Recall from \cref{defn:pc} that we define the log-parameter that corresponds to the edge $(n,c)$ as $\phi_{n,c} := \log \theta_{n,c}$, and the set of all parameters of a PC as $\params := \{\phi_{n,c}\}_{n,c}$.\footnote{We assume for simplicity that distributions of input nodes have no parameters (\eg indicator distributions). Our analysis can be easily extended to exponential family input distributions.} Since $\params$ does not necessarily define a normalized PC, we distinguish between the unnormalized and normalized forms of the model: let $\tilde{\p}_{\params} (\x)$ denote the unnormalized output of the circuit computed via the feedforward pass defined by \cref{eq:pc-fwd}, and define the normalized distribution as
    \begin{align*}
        \p_{\params} (\x) := \tilde{\p}_{\params} (\x) / Z (\params),
    \end{align*}
\noindent where $Z (\params)$ is the normalizing constant. We extend the single-sample EM objective in \cref{eq:m-step-obj} to the following, which is defined on a dataset $\data$:
    \begin{align*}
        Q_{\params}^{\data} (\params') := \frac{1}{\abs{\data}} \sum_{\x \in \data} \sum_{\z} \p_{\params} (\z \given \x) \cdot \log \p_{\params'} (\x, \z).
    \end{align*}
Our analysis is rooted in the following result.

\begin{prop}
\label{prop:em-general-form}
    Given a PC $\p_{\params}$ with log-parameters $\params$ (\cf Def.~\ref{defn:pc}) and a dataset $\data$, $Q_{\params}^{\data} (\params')$ equals the following up to a constant term irrelevant to $\params'$: 
    % \guy{why is $\log \p_{\params} (\x)$ included in the equation, if the equation is up to a constant term irrelevant to $\params'$? Usually you write that when you drop such terms as $\log \p_{\params} (\x)$}
        \begin{align}
            \frac{1}{\abs{\data}} \sum_{\x \in \data} \left \langle \frac{\partial \log \p_{\params} (\x)}{\partial \params}, \params' \right \rangle - \mathtt{KL}_{\params} (\params'),
            \label{eq:em-obj-equivalence}
        \end{align}
    \noindent where $\mathtt{KL}_{\params} (\params') := \kld \left ( \p_{\params} (\X, \Z) \,\|\, \p_{\params{\prime}} (\X, \Z) \right )$ is the KL divergence between $\p_{\params}$ and $\p_{\params'}$.
    % \guy{bit confusing that $\log \p_{\params} (\x)$ is included in this, while the text says that constant terms irrelevant to phi prime are omitted?}
\end{prop}

The proof follows \citet{kunstner2021homeomorphic} and is provided in \cref{appx:proof-em-general-form}. \cref{prop:em-general-form} reveals that the EM update can be interpreted as maximizing a \emph{regularized first-order approximation} of the log-likelihood. To see this, we rewrite \cref{eq:em-obj-equivalence} by adding terms irrelevant to $\params'$:
    \begin{align*}
        \frac{1}{\abs{\data}} \sum_{\x \in \data} \underbrace{\log \p_{\params} (\x) + \left \langle \frac{\partial \log \p_{\params} (\x)}{\partial \params}, \params' - \params \right \rangle}_{\mathtt{LinLL}_{\params}^{\x} (\params')} - \mathtt{KL}_{\params} (\params').
    \end{align*}
In the above equation, the term $\mathtt{LinLL}_{\params}^{\x}(\params')$ corresponds to the linearization of $\log \p_{\params'} (\x)$ around the current parameters $\params$, capturing the local sensitivity of the log-likelihood to parameter changes. The KL term, $\mathtt{KL}_{\params}(\params')$, acts as a regularizer that penalizes large deviations in the joint distribution over $\X$ and $\Z$.

According to \cref{prop:em-general-form}, solving for the updated parameters $\params'$ requires computing two key quantities: the gradient of the log-likelihood $\partial \log \p_{\params} (\x) / \partial \params$, and the KL divergence $\mathtt{KL}_{\params} (\params')$. To express these terms in closed form, we introduce the concept of top-down probabilities, which is defined by \citet{dang2022sparse}.
% \guy{shouldn't this cite the pruning and growing paper?}

\begin{defn}[TD-prob]
\label{defn:td-probs}
    Given a PC $\p$ parameterized by $\params$, we define the top-down probability $\mathtt{TD} (n)$ of a node $n$ recursively from the root node to input nodes:
        \begin{align*}
            \mathtt{TD} (n) := \begin{cases}
                1 & n \text{~is~the~root~node}, \\
                \sum_{m \in \pa (n)} \mathtt{TD} (m) & n \text{~is~a~sum~node}, \\
                \sum_{m \in \pa (n)} \theta_{m,n} \cdot \mathtt{TD} (m)\!\! & n \text{~is~a~product~node},
            \end{cases}
        \end{align*}
    \noindent where $\theta_{m,n} := \exp (\phi_{m,n})$ and $\pa(n)$ is the set of parent nodes of $n$. Define the TD-prob of $\phi_{m,n}$ as $\mathtt{TD} (\phi_{m,n}) := \theta_{m,n} \cdot \mathtt{TD} (m)$, and denote by $\mathtt{TD} (\params)$ the vector containing the TD-probs of all edge parameters in the circuit.
\end{defn}

Intuitively, the TD-prob of a parameter quantifies how much influence it has on the overall output of the PC, and in particular, on the normalizing constant $Z (\params)$.\footnote{This can be observed from the fact that $Z (\params)$ can be computed via the same feedforward pass (Eq.~(\ref{eq:pc-fwd})) except that we set the output of input nodes to $1$.} We continue to express the two key terms in \cref{prop:em-general-form} in closed form.

\begin{lem}
\label{lem:key-terms}
    Assume the distributions defined by all nodes in a PC are normalized. For every $\x$, we have:
        \begin{align*}
            \text{(i)~} & \partial \log \p_{\params} (\x) / \partial \params  = \partial \log \tilde{\p}_{\params} (\x) / \partial \params - \mathtt{TD} (\params), \\
            \text{(ii)~} & \mathtt{KL}_{\params} (\params') = - \langle \mathtt{TD} (\params), \params' \rangle + C,
        \end{align*}
        % \guy{where it says p-hat I think it should say p-tilde?}
    \noindent where $C$ is a constant term independent of $\params'$.
\end{lem}

The assumption that the PC is normalized is mild and practical. In \cref{sec:em-vs-sgd} and \cref{appx:global-renorm}, we introduce a simple and efficient algorithm that adjusts the PC parameters to ensure normalization without affecting the structure of the circuit. We can now substitute the closed-form expressions for the gradient and the KL divergence into the general EM objective $Q_{\params}^{\data} (\params')$, which converts the problem into\footnote{We drop all terms that are independent of $\params'$.}
    \begin{align*}
        \left \langle \frac{1}{\abs{\data}} \sum_{\x \in \data} \frac{\partial \log \tilde{\p}_{\params} (\x)}{\partial \params} - \cancel{\mathtt{TD} (\params)}, \params' \right \rangle + \cancel{\langle \mathtt{TD} (\params), \params' \rangle}.
    \end{align*}
If we additionally require each node in the PC to define a normalized distribution, we impose the constraint $\sum_{c \in \ch(n)} \exp (\phi'_{n,c}) = 1$ for all sum nodes $n$. Incorporating these constraints into the EM objective results in a constrained maximization problem that has the following solution for every edge $(n,c)$ (see Appx.~\ref{appx:proof-constraint-em} for the derivation):
    \begin{align}
        \phi'_{n,c} = \log \theta'_{n,c}, \quad \theta'_{n,c} = \mathtt{F}_{\params}^{\data} (n, c) / Z,
        \label{eq:full-batch-em-params}
    \end{align}
\noindent where we define $\mathtt{F}_{\params}^{\data} (n, c) := \frac{1}{\abs{\data}} \sum_{\x \in \data} \frac{\partial \log \tilde{\p}_{\params} (\x)}{\partial \phi_{n,c}}$,\footnote{$\mathtt{F}_{\params}^{\data} (n, c)$ is to the PC flows defined by \citet{choi2021group}.} and $Z = \sum_{c' \in \ch(n)} \mathtt{F}_{\params}^{\data} (n, c')$ ensures that $n$ is normalized. 

While this full-batch EM algorithm in \cref{eq:full-batch-em-params} has been derived in prior work \citep{choi2021group,peharz2015foundations}, we recover it here through \cref{prop:em-general-form}. This paves the way for a principled mini-batch EM algorithm by generalizing the objective $Q_{\params}^{\data} (\params')$, as shown in the next section.

% \guy{use might be confused what is novel here. tell them that we recover the well-known full-batch EM algorithm for PCs this way; the point is that we arrived at it though Proposition 1, which now we can use to derive a novel bound for the mini-batch case. If you don't explain this, the reader will be confused as to what the point was.}

\subsection{Extension to the Mini-Batch Case}
\label{sec:mini-batch-em}

When the dataset is large, full-batch EM becomes inefficient and impractical as it requires scanning the entire dataset before making any parameter updates. In such cases, we instead wish to update the parameters after processing only a small subset of data points, which is commonly referred to as a mini-batch. Given a mini-batch of samples $\data$, \citet{peharz2020einsum} propose to update the parameters following:
    \begin{align}
        \theta'_{n,c} = (1 - \alpha) \cdot \theta_{n,c} + \alpha \cdot \mathtt{F}_{\params}^{\data} (n, c) / Z, \; (\alpha \in (0, 1))
        \label{eq:old-mini-em}
    \end{align}
\noindent where we borrow notation from \cref{eq:full-batch-em-params}. However, it remains unclear whether this update rule is grounded in a principled EM objective. In the following, we show that from the full-batch EM derivation, we can derive a mini-batch update rule that closely resembles the above, but with a crucial difference.

\cref{prop:em-general-form} expresses the EM objective as the sum of two terms: a linear approximation of the log-likelihood and a regularization term that penalizes deviation from the current model via KL divergence. When using only a mini-batch of samples, the log-likelihood may overlook parts of the data distribution not covered by the sampled subset. 

To account for this, we can put a weighting factor $\gamma > 1$ on the KL divergence (\ie $\mathtt{KL}_{\params} (\params')$ becomes $\gamma \cdot \mathtt{KL}_{\params} (\params')$).\footnote{Note that this is equivalent to $Q_{\params}^{\data} (\params') - (\gamma - 1) \cdot \mathtt{KL}_{\params} (\params')$ according to \cref{prop:em-general-form}.} Plugging in \cref{lem:key-terms} and dropping terms independent to $\params'$, the adjusted objective (\ie $Q_{\params}^{\data} (\params')$ with the additional weighting $\gamma$) becomes
    \begin{align*}
        \left \langle \mathtt{F}_{\params}^{\data}, \params' \right \rangle + (\gamma - 1) \cdot \langle \mathtt{TD} (\params), \params' \rangle,
    \end{align*}
\noindent where $\mathtt{F}_{\params}^{\data}$ collects all entries $\mathtt{F}_{\params}^{\data} (n, c)$ into a single vector, with each $\mathtt{F}_{\params}^{\data} (n, c)$ representing the aggregated gradient \wrt $\phi_{n,c}$. With the constraints that ensure each PC node defines a normalized distribution (\ie for each sum node $n$, $\sum_{c \in \ch(n)} \theta'_{n,c} \!=\! 1$), the solution is
    \begin{align}
        \theta'_{n,c} = \left ( \mathtt{TD}_{\params} (n) \cdot \theta_{n,c} + \eta \cdot F_{\params}^{\data} (n, c) \right ) / Z,
        \label{eq:ours-mini-em}
    \end{align}
\noindent where $\eta := 1 / (\gamma - 1)$ is the learning rate, $\mathtt{TD}_{\params} (n)$ is the TD-prob of node $n$ (Def.~\ref{defn:td-probs}), and $Z$ is a normalizing constant. The derivation is deferred to \cref{appx:proof-constraint-em}. In practice, compared to the full-batch EM update (Eq.~(\ref{eq:full-batch-em-params})), the only additional computation required is $\mathtt{TD}_{\params} (n)$, which can be efficiently implemented using any autodiff library to compute the gradient of $Z(\params)$ \wrt the log-parameters $\phi_{n,c}$ (see proof in Appx.~\ref{appx:proof-tdp-ad}).

To build intuition for the update rule, we consider the case where $\data$ contains only a single sample $\x$. In this setting, the update direction $\mathtt{F}_{\params}^{\x} (n, c)$ can be decomposed using the chain rule of derivatives:
    \begin{align*}
        \mathtt{F}_{\params}^{\x} (n, c) = \frac{\partial \log \tilde{\p}_{\params} (\x)}{\partial \phi_{n,c}} = \underbrace{\frac{\partial \log \tilde{\p}_{\params} (\x)}{\partial \log \tilde{\p}_{\params}^{n} (\x)}}_{\mathtt{F}_{\params}^{\x} (n)} \cdot \underbrace{\frac{\partial \log \tilde{\p}_{\params}^{n} (\x)}{\partial \phi_{n,c}}}_{\hat{\mathtt{F}}_{\params}^{\x} (n, c)}, 
    \end{align*}
\noindent where we define $\log \tilde{\p}_{\params}^{n} (\x)$ as the (unnormalized) log-likelihood of node $n$. A key observation is that the second term $\hat{\mathtt{F}}_{\params}^{\x} (n, c)$ is normalized \wrt all children of sum node $n$: $\sum_{c \in \ch(n)} \hat{\mathtt{F}}_{\params}^{\x} (n, c) = 1$ (see Appx.~\ref{appx:proof-normed-flows} for the proof). Intuitively, we now break down $\mathtt{F}_{\params}^{\x} (n, c)$ into the \emph{importance of node $n$} to the overall output (\ie $\mathtt{F}_{\params}^{\x} (n)$) and the \emph{relative contribution of child $c$} to $n$ (\ie $\hat{\mathtt{F}}_{\params}^{\x} (n, c)$). This simplifies \cref{eq:ours-mini-em} as
    \begin{align*}
        \theta'_{n,c} = \big ( \theta_{n,c} + \eta \cdot \mathtt{rel}_{\params}^{\x} (n) \cdot \hat{\mathtt{F}}_{\params}^{\x} (n, c) \big ) / Z,
    \end{align*}
\noindent where $\mathtt{rel}_{\params}^{\x} (n) := \mathtt{F}_{\params}^{\x} (n) / \mathtt{TD}_{\params} (n)$ can be viewed as the relative importance of $n$ to the PC output given $\x$. The term $\eta \cdot \mathtt{rel}_{\params}^{\x} (n)$ then acts as an adaptive learning rate for updating the child parameters of $n$, scaling the update magnitude according to how influential $n$ is for $\x$. In comparison, the mini-batch algorithm in \cref{eq:old-mini-em} uses a fixed learning rate for all parameters.
% \guy{this is very nice; make it even clearer by contrasting explicitly with the equation at the start of 3.2 and giving that one a number.}

This difference is reflected in the example shown in \cref{fig:mini-em-example}. Given the PC on the left, which represents a mixture of two Gaussians (middle). Suppose we draw one sample $x \!=\! -1.5$. This sample does not reflect the full distribution and only activates the left mode. Our algorithm accounts for this by assigning a small effective learning rate to node $n_2$ ($\mathtt{rel}_{\params}^{x} (n_2) \!\approx\! 0.0004$) as it is ``not responsible'' for explaining this particular input and focuses the update on $n_1$ ($\mathtt{rel}_{\params}^{x} (n_1) \!\approx\! 1.9996$). In contrast, \cref{eq:old-mini-em} applies a uniform learning rate across all parameters, leading it to also update $n_2$ unnecessarily to fit the current sample.

\begin{figure}[t]
    \centering
    \includegraphics[width=\linewidth]{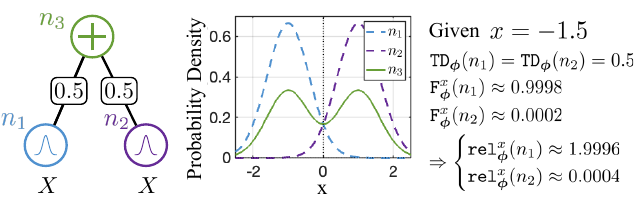}
    \vspace{-2.4em}
    \caption{The proposed algorithm implicitly applies an adaptive learning rate to each node. For the PC shown on the left, given a sample $x \!=\! -1.5$, the algorithm uses a large learning rate to update $n_1$ while keeping $n_2$ almost unchanged.}
    \label{fig:mini-em-example}
    \vspace{-0.4em}
\end{figure}

\section{Connections with Gradient-Based Optimizers}
\label{sec:em-vs-sgd}

Gradient-based optimization methods, such as stochastic gradient descent (SGD), can be interpreted through a lens similar to the EM formulation. Recall from \cref{prop:em-general-form} that the EM objective comprises a linear approximation of the log-likelihood, along with a KL divergence regularizer that penalizes deviations from the current model. In contrast, standard gradient-based updates can be viewed as maximizing the same linear approximation of the log-likelihood, but with an \emph{L2 regularization} penalty on parameter updates instead of a KL divergence:
    \begin{align*}
        \frac{1}{\abs{\data}}\! \sum_{\x \in \data}\! \log \p_{\params} (\x) \!+\! \left \langle \frac{\partial \log \p_{\params} (\x)}{\partial \params}, \params' \!- \params \right \rangle \!-\! \gamma \| \params' \!- \params \|_{2}^{2}.
    \end{align*}
Solving for $\params'$ gives $\params' = \params + \eta \cdot \partial \log \p_{\params} (\x) / \partial \params$, a standard gradient ascent step with $\eta = 1 / (2 \gamma)$. 

\boldparagraph{The Regularization Terms.}
While this analogy highlights a shared structure between EM and gradient-based optimizers, it also reveals a fundamental discrepancy. The KL divergence in the EM formulation is a natural measure of distance between distributions, while the L2 penalty in gradient-based updates only constrains the movement in parameter space. This distinction is important because proximity in parameter space does not necessarily translate to proximity in distribution space. For instance, adding all parameters $\params$ (note that they represent log-probabilities) by a constant leaves the distribution unchanged, yet the L2 penalty would still register this as a large deviation. Therefore, compared to gradient-based methods, the EM formulation better respects the geometry of distributions.

\boldparagraph{Normalization Constraints.}
Another notable difference between gradient-based methods and the EM algorithms (both the full-batch and the mini-batch ones) is the treatment of normalization constraints. In EM, parameters are updated in a way that preserves local normalization, \ie the edge parameters of each sum node $n$ are guaranteed to sum to one after every update. On the other hand, standard gradient-based optimization does not enforce this constraint, and thus intermediate parameter values may fall outside the normalized parameter space.

One might naturally wonder whether this discrepancy leads to better or worse training dynamics, since optimization would be carried out in an enlarged parameter space. Perhaps surprisingly, we show that it has \emph{no} effect on the optimization process, as the parameters $\params$ can always be mapped to a locally normalized counterpart $\params’$ that preserves both the represented distribution and the gradients with respect to the parameters.

Specifically, given a PC with parameters $\params$ (corresponds to $\boldsymbol{\theta} := \exp (\params)$), there exists a normalization algorithm that outputs $\params'$ and ensures (i) all parameters are locally normalized (\ie $\forall n, \sum_{c \in \ch(n)} \theta_{n,c} = 1$), (ii) the represented distributions are unchanged (\ie $\forall \x, \tilde{\p}_{\params} (\x) \propto \tilde{\p}_{\params'} (\x)$), and (iii) the gradients with respect to the parameters are preserved (\ie $\forall \x, \partial \tilde{\p}_{\params} (\x) / \partial \params = \partial \tilde{\p}_{\params'} (\x) / \partial \params'$.

Intuitively, the above three conditions guarantee that applying this normalization algorithm neither alters the model's probabilistic semantics nor interferes with the optimization dynamics. In other words, the model continues to represent the same distribution, and the gradients used for subsequent updates remain the same. Therefore, it can be applied after each gradient update without changing the learning dynamics.

In the following, we present the normalization algorithm and show how it can be implemented efficiently.\footnote{The existence of this normalization algorithm has been previously shown by \citet{martens2014expressive}, although they did not provide an explicit algorithm.} A detailed analysis and formal proofs of the three properties are provided in \cref{appx:global-renorm}.

\begin{figure*}
    \includegraphics[width=\textwidth]{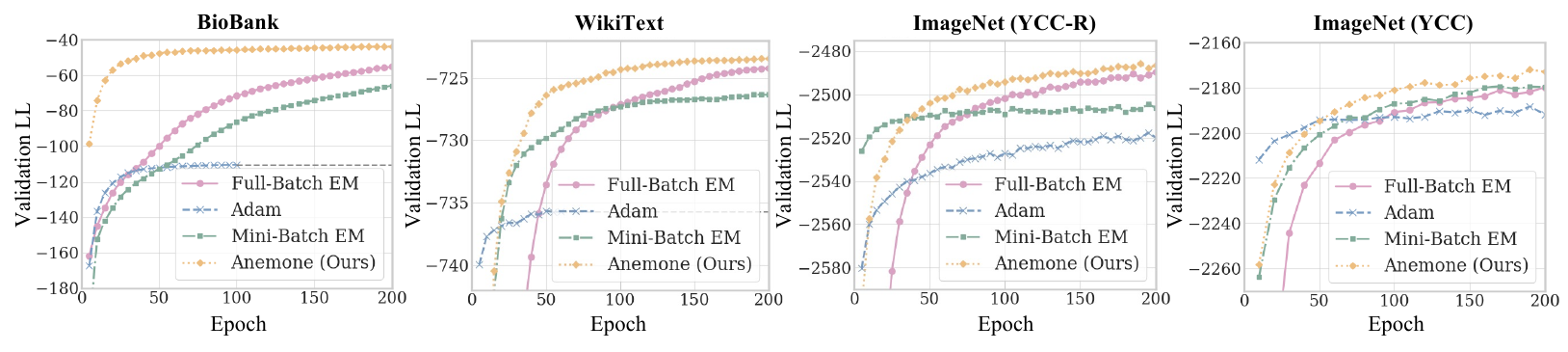}
    \vspace{-2.2em}
    \caption{\textbf{Log-Likelihood over epochs on four diverse datasets.} For the ImageNet (YCC-R) and ImageNet (YCC) datasets, an HCLT with hidden size 512 is used; for the WikiText dataset, an HMM with hidden size 256 is used; for the BioBank dataset, a PDHCLT with hidden size 1024 is used. An\textbf{em}one achieves significantly faster convergence as well as final LL across all four cases.}
    \label{fig:learning-curves}
\end{figure*}

\begin{table*}[ht]
\centering
\caption{\textbf{Negative LLs on the UK BioBank Chromosome 6 dataset. } An\textbf{em}one consistently and significantly outperforms all baseline across every tested model architecture. The results also show a clear performance hierarchy among the baseline methods. The best result in each column is marked in bold.}
\vspace{-0.9em}
\label{tab:density_estimation_dna}
\scalebox{0.9}{
\begin{tabular}{@{}l cccc@{}}
\toprule
\multirow{2}{*}[-0.3em]{\textbf{Optimizer}} & \multicolumn{4}{c}{\textbf{BioBank Chr6}} \\
\cmidrule(lr){2-5}
& \textbf{HCLT 512} & \textbf{HCLT 1024} & \textbf{PDHCLT 512} & \textbf{PDHCLT 1024}\\
\midrule
Full-batch EM & 55.3 & 53.8 & 46.5 & 45.1 \\
Adam & 102.7 &	100.4 & 112.4 & 110.4 \\
Mini-batch EM & 55.7 & 55.5 & 49.5 & 47.2 \\
An\textbf{em}one (Ours) & \textbf{54.5} & \textbf{52.1} & \textbf{45.3} & \textbf{42.2} \\
\bottomrule
\end{tabular}}
\vspace{-1.0em}
\end{table*}

\begin{table}[ht]
\centering
\caption{\textbf{Convergence speed (epochs) for PDHCLT 1024 on BioBank Chr6 Dataset.} The table reports epochs to reach specific LL thresholds, with $\Delta$ representing the difference from the best LL of -45.1 before an\textbf{em}one (\cref{tab:density_estimation_dna}). Lower is better. Bold marks the best result per column; $\infty$ indicates failure to reach the threshold in time.}
\vspace{-0.8em}
\label{tab:convergence_speed}
\setlength{\tabcolsep}{4pt} 
\begin{tabular}{@{}l ccc@{}}
\toprule
% Headers have been reordered
\textbf{Method} & \textbf{LL $\geq$ -48} & \textbf{LL $\geq$ -46} & \textbf{LL $\geq$ -45.1} \\
& \small($\Delta \approx 2.9$) & \small($\Delta \approx 0.9$) & \small($\Delta =0$) \\
\midrule
Full EM & 450 & 645 & 1000 \\
Adam & $\infty$ & $\infty$ & $\infty$ \\
Mini EM & 715 & $\infty$ & $\infty$ \\
An\textbf{em}one & \textbf{50} & \textbf{80} & \textbf{130} \\
\bottomrule
\end{tabular}
\vspace{-1.0em}
\end{table}

\boldparagraph{The Algorithm.}
The normalization procedure consists of two simple passes over the PC. First, we perform a feedforward evaluation of the PC to compute the partition function $Z_{n} (\params)$ at every node $n$:
    \begin{align*}
        Z_{\params} (n) \!=\! \begin{cases}
            1 & n \text{~is~an~input~node}, \\
            \prod_{c \in \ch(n)} Z_{\params} (c) & n \text{~is~a~product~node}, \\
            \sum_{c \in \ch(n)} \exp (\phi_{n,c}) \!\cdot\! Z_{\params} (c) \!\!\!\!\! & n \text{~is~a~sum~node}.
        \end{cases}
    \end{align*}
Next, for every edge $(n,c)$ where $n$ is a sum node and $c \in \ch(n)$, we update the parameter as
    \begin{align}
        \phi'_{n,c} = \log \left ( \frac{\theta_{n,c} \cdot Z_{\params} (c)}{Z_{\params} (n)} \right ).
        \label{eq:renorm}
    \end{align}
% \noindent where $\theta_{n,c} := \exp (\phi_{n,c})$.

\section{Experiments}

In this section, we empirically evaluate an\textbf{em}one against existing EM- and gradient-based optimizers across a range of PC models and datasets. \cref{sec:exp-setup} introduces the models, datasets, and baseline optimizers. In \cref{sec:exp-overall}, we ask whether an\textbf{em}one can achieve higher log-likelihoods at convergence compared to existing approaches, and whether an\textbf{em}one converges faster. Finally, \cref{sec:exp-ablation} investigates the key design factors in setting hyperparameters through a series of ablation studies.

\subsection{Experimental Setup and Baselines}
\label{sec:exp-setup}

We conduct our empirical evaluation across three distinct domains (\ie DNA sequence, image, and text) using various PC architectures. 

\paragraph{DNA Sequence Modeling.}
We evaluate density estimation performance on a high-dimensional genomics dataset from the UK Biobank \citep{bioBank}. For this task, we train Hidden Chow-Liu Trees (HCLTs) \citep{liu2021tractable} and their variant that combines the PD structure \citep{poon2011sum} termed Partitioned Data HCLTs (PDHCLTs), which is defined in \cref{appx:pdhclt}. A detailed description of the dataset and all model configurations is provided in \cref{appx:datasets}.

\boldparagraph{Image Modeling.}
We adopt the ImageNet32 dataset \citep{deng2009imagenet} with two color transformations, \ie a lossy YCC transformation and its lossless variant YCC-R \citep{malvar2003ycocg}. Each 32$\times$32 image is partitioned into four 16$\times$16 image patches, which results in a total of $16 \times 16 \times 3 = 768$ variables. We employ HCLTs and PDHCLTs for both datasets. Details of the datasets and the models are included in \cref{appx:datasets}.

\boldparagraph{Language Modeling.}
We use the WikiText-103 \citep{merity2017pointer} dataset, which is widely used for language modeling. The dataset is preprocessed by the GPT-2 tokenizer \citep{radford2019language}. We evaluate the Hidden Markov Model (HMM) PC architecture and its recently proposed variant Monarch HMM \citep{zhang2025scaling} on it.

\boldparagraph{EM-Based Baselines.}
We adopt two EM baselines, which are the standard full-batch EM and the mini-batch EM proposed by \citet{peharz2020einsum} (\cf Eq.~(\ref{eq:old-mini-em})). For the full-batch EM, there are no hyperparameters to tune. In contrast, for the mini-batch EM, we tune the batch size in the range $\{512, 16384\}$ and the step size $\alpha$ in Eq.~(\ref{eq:old-mini-em}) over $\{0.05, 0.1, 0.4\}$. 

\begin{table*}[ht]
\centering
\caption{\textbf{Negative LLs on the ImageNet32 dataset's validation subset.} An\textbf{em}one consistently outperforms other baselines when training HCLT models. We observe a general performance ranking where Adam optimizer is the weakest, followed by Mini-batch EM and Full-batch EM. OOM denotes out-of-memory errors due to computing resource limitations. The best result in each column is marked in bold.}
\vspace{-0.8em}
\label{tab:density_estimation_img}
\scalebox{0.9}{
\begin{tabular}{@{}l ccc cc@{}}
\toprule
\multirow{2}{*}[-0.3em]{\textbf{Optimizer}} & \multicolumn{3}{c}{\textbf{ImageNet32 YCC-R}} & \multicolumn{2}{c}{\textbf{ImageNet32 YCC}} \\
\cmidrule(lr){2-4} \cmidrule(lr){5-6}
& \textbf{PDHCLT 256} & \textbf{HCLT 512} & \textbf{HCLT 1024} & \textbf{HCLT 512} & \textbf{HCLT 1024} \\
\midrule
Full-batch EM & \textbf{2529} & 2480 & \textbf{2469} & 2164 & 2163 \\
Adam & 2553 & 2518 & OOM & 2187 & OOM \\
Mini-batch EM & \textbf{2529} & 2506 & 2470 & 2179 & 2232 \\
An\textbf{em}one (Ours) & 2530 & \textbf{2477} & \textbf{2469} & \textbf{2158} & \textbf{2159} \\
\bottomrule
\end{tabular}}
\end{table*}

% we propose to apply a momentum update to the flows. Specifically, instead of directly using the flows $\mathtt{F}_{\params} (n,c)$ for the optimization step, we update them using an exponential moving average, \ie
%     \begin{align*}
%         \hat{\mathtt{F}}_{\params} (n,c) \leftarrow \Big ( (1 - \beta) \cdot \hat{\mathtt{F}}_{\params} (n,c) + \beta \cdot \mathtt{F}_{\params} (n,c) \Big ) / (1 - \beta^{t}),
%     \end{align*}
% \noindent where $\hat{\mathtt{F}}_{\params} (n,c)$ is the new flow estimate used for update (initialized to $0$), $\beta \in [0,1)$ is the momentum coefficient, and $t$ is the total number of iterations.

\begin{table*}[ht]
\centering
\caption{\textbf{Negative LLs on the WikiText-103 dataset.} An\textbf{em}one achieves the best LLs on three of the four tested model configurations and is highly competitive on the fourth. The best results are marked in bold.}
\vspace{-0.9em}
\label{tab:density_estimation_text}
\scalebox{0.9}{
\begin{tabular}{@{}l cccc@{}}
\toprule
\multirow{2}{*}[-0.3em]{\textbf{Optimizer}} & \multicolumn{4}{c}{\textbf{WikiText}} \\
\cmidrule(lr){2-5}
& \textbf{HMM 256} & \textbf{HMM 512} & \textbf{HMM 1024} & \textbf{Monarch HMM 1024}\\
\midrule
Full-batch EM & 722.6 & 702.2 & 682.8 & 738.1 \\
Adam & 735.7 & OOM & OOM & OOM \\
Mini-batch EM & 725.2 & 703.2 & \textbf{682.1} & 734.6 \\
An\textbf{em}one (Ours) & \textbf{722.2} & \textbf{701.7} & 682.3 & \textbf{734.0} \\
\bottomrule
\end{tabular}}
\vspace{-0.6em}
\end{table*}

\boldparagraph{Gradient-Based Baselines.}
We adopt the Adam optimizer \citep{adam2014method}, which is used by many recent works. We tune the batch size and the learning rate in the ranges $\{512, 1024\}$ and $\{1\times 10^{-2}, 3\times 10^{-3}, 5\times 10^{-3}, 1\times 10^{-3}\}$, respectively.

\boldparagraph{Our Method.}
For ease of definition, we express the step size of our method as $\alpha = \eta / (\eta - 1)$, where $\eta$ is given in \cref{eq:ours-mini-em}. The hyperparameter search is detailed in \cref{appx:optimizers}. Similar to the mini-batch EM algorithm, we use the proposed momentum update with $\beta = 0.9$.
\label{sec:momentum}
We propose to apply a momentum update to the flows $\mathtt{F}_{\params}^{\data} (n,c)$. Specifically, we initialize the momentum flows $\mathtt{Fm}_{\params}^{\data} (n,c) = \mathbf{0}$, then before every EM step, we update $\mathtt{Fm}_{\params}^{\data} (n,c)$ following:
    \begin{align*}
        \mathtt{Fm}_{\params}^{\data} (n,c) \leftarrow \beta \cdot \mathtt{Fm}_{\params}^{\data} (n,c) + (1-\beta) \cdot \mathtt{F}_{\params}^{\data} (n,c).
    \end{align*}
Finally, we replace the $\mathtt{F}_{\params}^{\data} (n,c)$ in \cref{eq:old-mini-em} with $\mathtt{Fm}_{\params}^{\data} (n,c) / (1 - \beta^{T+1})$, where $T$ is the number of updates performed. We compare the performance with and without the momentum update for both an\textbf{em}one and mini-batch EM in \cref{sec:exp-ablation}.

\subsection{Overall Performance and Convergence}
\label{sec:exp-overall}

We begin by examining the training dynamics of different optimizers. \cref{fig:learning-curves} shows the training curves (\ie validation LL vs. number of epochs) across all four datasets, each paired with an appropriate PC architecture (see the caption for the details).

We start by focusing on the three baseline approaches, \ie full-batch EM, Adam, and mini-batch EM. A consistent trend is that full-batch EM consistently reaches better (or comparable) performance at convergence. This is a strong indicator that the existing mini-batch optimizers (Adam and mini-batch EM) fail to compensate for the reduced information available in a mini-batch compared to the full dataset.

In contrast, despite being a mini-batch algorithm, an\textbf{em}one achieves significantly better performance at convergence, even compared to full-batch EM. We evaluate the final log-likelihoods of each optimizer across a wider range of PC architectures. The final log-likelihoods, evaluated either at convergence or after a fixed maximum number of epochs, are shown in \cref{tab:density_estimation_dna,tab:density_estimation_img,tab:density_estimation_text} for the DNA, image, and text modeling tasks, respectively. On the BioBank Chr6 dataset, an\textbf{em}one achieves consistent and significant performance gains over all baselines. For the ImageNet32 and WikiText datasets, it again obtains better results in the majority of cases.

% \anji{We further evaluate the LL at convergence for more PC architectures, refer to the three tables (note that the numbers are hidden sizes).}

Additionally, as indicated by \cref{fig:learning-curves}, an\textbf{em}one converges faster than all baselines, including the mini-batch ones that are designed for faster convergence. We further quantify the number of epochs used to reach a certain LL. Specifically, as shown in \cref{tab:convergence_speed} an\textbf{em}one requires $\sim$8x fewer epochs to reach the same validation LL. Further experiments are deferred to \cref{appx:exp-convergence}. 

% \anji{converges faster. We first talk about this using \cref{fig:learning-curves}. We then refer to \cref{tab:convergence_speed} for even more impressive results and also refer to the appendix table \cref{tab:convergence_speed_wikitext}.}

% Across all three diverse domains that we evaluate, our proposed optimizer, an\textbf{em}one, demonstrates robust and consistently superior performance.
% The 

\subsection{Ablation Study}
\label{sec:exp-ablation}

To disentangle the performance gains of an\textbf{em}one from the general effects of momentum (cf. \cref{sec:momentum}), we conduct ablation study shown in \cref{fig:ablation_momentum}. While the results confirm that momentum improves the final log-likelihood for an\textbf{em}one (left panel), they also show that it provides little benefit when applied to vanilla mini-batch EM (right panel), suggesting that the improvements are not from momentum alone, but from the synergy between momentum and an\textbf{em}one, that is not observed in standard mini-batch approaches.

\begin{figure}[t]
    \centering
    \includegraphics[width=\linewidth]{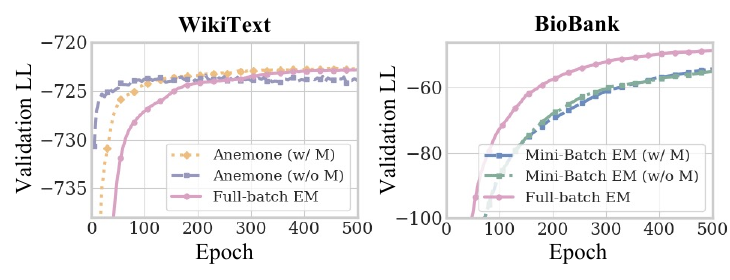}
    \vspace{-2.2em}
    \caption{\textbf{Ablation study on the effect of momentum when combined with an\textbf{em}one and vanilla mini-batch EM, respectively.} \textbf{Left: } For an\textbf{em}one optimizer (HMM with hidden size of 256 on WikiText), incorporating momentum improves the final log-likelihood despite slightly slower initial convergence, while still being significantly faster than  full-batch EM. \textbf{Right: }In contrast, for mini-batch EM (PDHCLT with hidden size of 512 on BioBank Chr6), momentum provides little benefit.}
    \label{fig:ablation_momentum}
\end{figure}

\section{Related Works and Conclusion}

\boldparagraph{Modeling Advancements in PCs.}
The development of PCs has been marked by significant progress in enhancing their expressiveness and utility as generative models. Since the initial establishment of PCs, research has focused on developing more expressive and scalable PC structures. Specifically, a line of research has sought to design PC structures that are expressive yet parameter-efficient \citep{rahman2014cutset,adel2015learning,liu2021tractable}, while another set of approaches pursues iterative structure learning strategies that progressively expand the model's capacity \citep{liang2017learning,dang2022sparse,di2021random,liu2023understanding}. Both directions have contributed to significant performance gains on widely used text and image datasets.

\boldparagraph{Parameter Learning of PCs.}
Beyond structure, a central challenge in learning PCs lies in the optimization of their parameters. This problem has been studied from both the algorithmic and the systems perspective. On the algorithmic side, two families of approaches dominate: EM-style updates and gradient-based methods. The EM algorithm was first applied to PCs by \citet{poon2011sum}, and later extended to mini-batch training in \citet{peharz2020einsum}. Gradient-based optimizers such as Adam \citep{kingma2014adam} have also become a common choice in practice due to their simplicity and scalability.

On the systems side, considerable effort has gone into developing efficient implementations that can handle the large computational and memory demands of PCs. Optimized einsum backends \citep{peharz2020einsum}, specialized libraries such as SPFlow \citep{molina2019spflow}, and more recently PyJuice \citep{liu2024scaling} provide high-performance primitives that enable scaling to PCs with billions of parameters.

\boldparagraph{Conclusion.}
This work addresses the underperformance of mini-batch optimizers for PCs. We identify that existing methods ineffectively regularize distribution changes, causing them to ``overfit'' to the current mini-batch. We propose anemone, a novel mini-batch EM algorithm that applies an implicit adaptive learning rate to each parameter, scaled by its contribution to the batch likelihood. Across extensive experiments, anemone consistently outperforms existing optimizers in both convergence speed and final performance.

\section*{Acknowledgements}

This work was funded in part by the National University of Singapore under its Start-up Grant (Award No: SUG-251RES2505), the DARPA ANSR, CODORD, and SAFRON programs under awards FA8750-23-2-0004, HR00112590089, and HR00112530141, NSF grant IIS1943641, and gifts from Adobe Research, Cisco Research, Qualcomm, and Amazon. Approved for public release; distribution is unlimited.

\bibliography{aistats}
% \section{Conclusion}

%%%%%%%%%%%%%%%%%%%%%%%%%%%%%%%%%%%%%%%%%%%%%%%%%%%%%%%%%%%%
\section*{Checklist}

\begin{enumerate}

  \item For all models and algorithms presented, check if you include:
  \begin{enumerate}
    \item A clear description of the mathematical setting, assumptions, algorithm, and/or model. [Yes]
    \item An analysis of the properties and complexity (time, space, sample size) of any algorithm. [Yes]
    \item (Optional) Anonymized source code, with specification of all dependencies, including external libraries. [Yes/No/Not Applicable]
  \end{enumerate}

  \item For any theoretical claim, check if you include:
  \begin{enumerate}
    \item Statements of the full set of assumptions of all theoretical results. [Yes]
    \item Complete proofs of all theoretical results. [Yes]
    \item Clear explanations of any assumptions. [Yes]     
  \end{enumerate}

  \item For all figures and tables that present empirical results, check if you include:
  \begin{enumerate}
    \item The code, data, and instructions needed to reproduce the main experimental results (either in the supplemental material or as a URL). [Yes]
    \item All the training details (e.g., data splits, hyperparameters, how they were chosen). [Yes]
    \item A clear definition of the specific measure or statistics and error bars (e.g., with respect to the random seed after running experiments multiple times). [No]
    \item A description of the computing infrastructure used. (e.g., type of GPUs, internal cluster, or cloud provider). [Yes]
  \end{enumerate}

  \item If you are using existing assets (e.g., code, data, models) or curating/releasing new assets, check if you include:
  \begin{enumerate}
    \item Citations of the creator If your work uses existing assets. [Yes]
    \item The license information of the assets, if applicable. [Yes]
    \item New assets either in the supplemental material or as a URL, if applicable. [Not Applicable]
    \item Information about consent from data providers/curators. [Yes]
    \item Discussion of sensible content if applicable, e.g., personally identifiable information or offensive content. [Not Applicable]
  \end{enumerate}

  \item If you used crowdsourcing or conducted research with human subjects, check if you include:
  \begin{enumerate}
    \item The full text of instructions given to participants and screenshots. [Not Applicable]
    \item Descriptions of potential participant risks, with links to Institutional Review Board (IRB) approvals if applicable. [Not Applicable]
    \item The estimated hourly wage paid to participants and the total amount spent on participant compensation. [Not Applicable]
  \end{enumerate}

\end{enumerate}

\clearpage
\appendix
\thispagestyle{empty}

% Supplementary material: To improve readability, you must use a single-column format for the supplementary material.
\onecolumn
\aistatstitle{Rethinking Probabilistic Circuit Parameter Learning: \\
Supplementary Materials}

\section{Structural Properties of PCs}
\label{appx:structural-props}

We provide formal definitions of smoothness and decomposability. Please refer to \citet{choi2020probabilistic} for a comprehensive overview.

\begin{defn}[Smoothness and Decomposability]
\label{defn:sm-dec}
    Define the scope $\scope (n)$ of a PC node $n$ as the set of all variables defined by its descendant input nodes. A PC $\p$ is smooth if for every sum node $n$, its children share the same scope: $\forall c_1, c_2 \in \ch(n)$, $\scope (c_1) = \scope (c_2)$. $\p$ is decomposable if for every product node $n$, its children have disjoint scopes: $\forall c_1, c_2 \in \ch(n)$ ($c_1 \neq c_2$), $\scope (c_1) \cap \scope (c_2) = \emptyset$.
\end{defn}

\section{Proofs}

This section provides proof of the theoretical results stated in the main paper.

\subsection{Interpreting the EM Algorithm of PCs}
\label{appx:proof-em-general-form}

This section provides the proof of \cref{prop:em-general-form}, which interprets the full-batch EM algorithm of PCs in a new context.

\begin{proof}[Proof of \cref{prop:em-general-form}]
    We begin by formalizing the latent-variable-model view of PCs. Given a PC $\p_{\params} (\X)$ parameterized by $\params$, we define a set of latent variables $\Z$ such that $\p_{\params} (\X) = \sum_{\z} \p_{\params} (\X, \Z = \z)$. Specifically, we associate a latent variable $Z_{n}$ with each sum node $n$ in the PC. We use $Z_{n} = i$ ($i \in \{1, \dots, \abs{\ch(n)}\}$) to denote that we ``select'' the $i$-th child node by zeroing out all the probabilities coming from all other child nodes:
        \begin{align*}
            \p_{n} (\x, \Z_n = i, \z_{\backslash n}) = \sum_{c \in \ch(n)} \exp(\phi_{n,c}) \cdot \p_{c} (\x, \Z_n = i, \z_{\backslash n}) \cdot \mathbbm{1} [c = c_i],
        \end{align*}
    \noindent where we define $c_i$ as the $i$-th child node of $n$, and $\Z_{\backslash n} := \Z \backslash Z_{n}$.

    We further show that $\p_{\params} (\X, \Z)$ is an exponential family distribution. To see this, it suffices to construct a set of $\abs{\params}$ sufficient statistics $S (\x, \z)$ such that for every $\x$ and $\z$, the likelihood can be expressed as:
        \begin{align*}
            \p_{\params} (\x, \z) = \exp \left ( \langle S (\x, \z), \params \rangle - A (\params) \right ),
        \end{align*}
    \noindent where $A (\params) = \log \sum_{\x, \z} \langle S (\x, \z), \params \rangle$ is the log partition function that normalizes the distribution. Note that $A (\params)$ is convex by definition.

    To construct $S (\x, \z)$, we first define the \emph{support} $\mathtt{supp} (n)$ of every node recursively as follows:
        \begin{align*}
            \mathtt{supp} (n) := \begin{cases}
                \{(\x, \z) : \p_{n} (\x) > 0\} & n \text{~is~an~input~node}, \\
                \bigcap_{c \in \ch(n)} \mathtt{supp} (c) & n \text{~is~a~product~node}, \\
                \bigcup_{c \in \ch(n)} \left ( \{(\x,\z) : z_{n} = c\} \cap \supp (c) \right ) & n \text{~is~a~sum~node},
            \end{cases}
        \end{align*}
    \noindent where $z_n = c$ means $z_n = i$ if $c$ is the $i$-th child of $n$.

    The sufficient statistics $S (\x, \z)$ can be defined using the support. Specifically, the sufficient statistics corresponding to the parameter $\phi_{n,c}$, denoted $S_{\phi_{n,c}} (\x, \z)$, can be represented as:
        \begin{align*}
            S_{\phi_{n,c}} (\x, \z) = \mathbbm{1} \Big [ (\x, \z) \in \mathtt{supp} (c) \text{~and~} z_{n} = c \Big ],
        \end{align*}
    \noindent where $\mathbbm{1} [\cdot]$ is the indicator function.

    Before proceeding, we define the Bregman divergence induced by a convex function $h$ as:
    \begin{align*}
        D_{h} (\y, \x) := h(\y) - h(\x) - \left \langle \frac{\partial h(\x)}{\partial \x}, \y - \x \right \rangle.
    \end{align*}

    The following part partially follows \citet{kunstner2021homeomorphic}. We plug in the exponential family distribution form of the PC into the definition of $Q_{\boldsymbol{\phi}} (\boldsymbol{\phi}')$:
    \begin{flalign*}
        && Q_{\boldsymbol{\phi}} (\boldsymbol{\phi}') & = \frac{1}{\abs{\data}} \sum_{\x \in \data} \sum_{\z} \p_{\boldsymbol{\phi}} (\z \given \x) \log \p_{\boldsymbol{\phi}'} (\x, \z), && \\
        && & = \frac{1}{\abs{\data}} \sum_{\x \in \data} \sum_{\z} \p_{\boldsymbol{\phi}} (\z \given \x) \left [ \left \langle S(\x, \z), \boldsymbol{\phi}' \right \rangle - A (\boldsymbol{\phi}') \right ], && \triangleright \text{Definition~of~} \p_{\boldsymbol{\phi}'} (\x, \z) \\
        && & = \frac{1}{\abs{\data}} \sum_{\x \in \data} \left \langle \expectation_{\p_{\boldsymbol{\phi}} (\z \given \x)} \left [ S(\x, \z) \right ], \boldsymbol{\phi}' \right \rangle - A(\boldsymbol{\phi}'). && \triangleright \text{Linearity~of~expectation}
    \end{flalign*}
    We then subtract both sides by $Q_{\boldsymbol{\phi}} (\boldsymbol{\phi})$, which is irrelevant to $\boldsymbol{\phi}'$:
    \begin{flalign}
        && Q_{\boldsymbol{\phi}} (\boldsymbol{\phi}') - Q_{\boldsymbol{\phi}} (\boldsymbol{\phi}) & = \frac{1}{\abs{\data}} \sum_{\x \in \data} \left \langle \expectation_{\z \sim \p_{\boldsymbol{\phi}} (\cdot \given \x)} \left [ S(\x, \z) \right ], \boldsymbol{\phi}' - \boldsymbol{\phi} \right \rangle - A(\boldsymbol{\phi}') + A(\boldsymbol{\phi}), && \nonumber \\
        && & = \frac{1}{\abs{\data}} \sum_{\x \in \data} \left \langle \expectation_{\z \sim \p_{\boldsymbol{\phi}} (\cdot \given \x)} \left [ S(\x, \z) \right ] - \frac{\partial A(\boldsymbol{\phi})}{\partial \boldsymbol{\phi}}, \boldsymbol{\phi}' - \boldsymbol{\phi} \right \rangle - \underbrace{\left ( A(\boldsymbol{\phi}') - A(\boldsymbol{\phi}) - \left \langle \frac{\partial A(\boldsymbol{\phi})}{\partial \boldsymbol{\phi}}, \boldsymbol{\phi}' - \boldsymbol{\phi} \right \rangle \right )}_{D_{A} (\boldsymbol{\phi}', \boldsymbol{\phi})}, \nonumber \\
        && & = \frac{1}{\abs{\data}} \sum_{\x \in \data} \left \langle \expectation_{\z \sim \p_{\boldsymbol{\phi}} (\cdot \given \x)} \left [ S(\x, \z) \right ] - \frac{\partial A(\boldsymbol{\phi})}{\partial \boldsymbol{\phi}}, \boldsymbol{\phi}' - \boldsymbol{\phi} \right \rangle - D_{A} (\boldsymbol{\phi}', \boldsymbol{\phi}). \label{eq:proof1-1}
    \end{flalign}
    We continue by simplifying the first term in the above expression. To do this, consider the gradient of $\LL (\boldsymbol{\phi})$ w.r.t. $\boldsymbol{\phi}$:
    \begin{align*}
        \frac{\partial \LL (\boldsymbol{\phi})}{\partial \boldsymbol{\phi}} & = \frac{1}{\abs{\data}} \sum_{\x \in \data} \frac{\partial \log \p_{\boldsymbol{\phi}} (\x)}{\partial \boldsymbol{\phi}}, \\
        & = \frac{1}{\abs{\data}} \sum_{\x \in \data} \frac{\partial \log \Big ( \sum_{\z} \exp \big ( \left \langle S(\x, \z), \boldsymbol{\phi} \right \rangle \big ) \Big )}{\partial \boldsymbol{\phi}} - \frac{\partial A(\boldsymbol{\phi})}{\partial \params}, \\
        & = \frac{1}{\abs{\data}} \sum_{\x \in \data} \sum_{\z} \frac{\exp \big ( \left \langle S(\x, \z), \boldsymbol{\phi} \right \rangle \big ) \cdot S(\x, \z)}{\sum_{\z'} \exp \left ( \left \langle S(\x, \z'), \boldsymbol{\phi} \right \rangle \right )} - \frac{\partial A(\boldsymbol{\phi})}{\partial \boldsymbol{\phi}}, \\
        & = \expectation_{\x \sim \data, \z \sim \p_{\boldsymbol{\phi}} (\cdot \given \x)} [S(\x, \z)] - \frac{\partial A(\boldsymbol{\phi})}{\partial \boldsymbol{\phi}}.
    \end{align*}
    Plug in \cref{eq:proof1-1}, we have
    \begin{align}
        Q_{\boldsymbol{\phi}} (\boldsymbol{\phi}') - Q_{\boldsymbol{\phi}} (\boldsymbol{\phi}) = \left \langle \frac{\partial \LL (\boldsymbol{\phi})}{\partial \boldsymbol{\phi}}, \boldsymbol{\phi}' - \boldsymbol{\phi} \right \rangle - D_{A} (\boldsymbol{\phi}', \boldsymbol{\phi}).
        \label{eq:proof1-2}
    \end{align}
    We proceed to demonstrate that $D_{A} (\boldsymbol{\phi}', \boldsymbol{\phi}) = D_{\mathrm{KL}} (\p_{\boldsymbol{\phi}} (\X, \Z) \,||\, \p_{\boldsymbol{\phi}'} (\X, \Z))$, where $D_{\mathrm{KL}} (p \,||\, q)$ is the KL divergence between distributions $p$ and $q$:
    \begin{flalign*}
        && D_{\mathrm{KL}} (\p_{\boldsymbol{\phi}} (\X, \Z) \,||\, \p_{\boldsymbol{\phi}'} (\x, \z)) & = \expectation_{\x, \z \sim \p_{\boldsymbol{\phi}}} \left [ \log \frac{\p_{\boldsymbol{\phi}} (\x, \z)}{\p_{\boldsymbol{\phi}'} (\x, \z)} \right ], \\
        && & = \expectation_{\x, \z \sim \p_{\boldsymbol{\phi}}} \left [ \left \langle S(\x, \z), \boldsymbol{\phi} - \boldsymbol{\phi}' \right \rangle \right ] + A(\boldsymbol{\phi}') - A(\boldsymbol{\phi}), \\
        && & = \left \langle \expectation_{\x, \z \sim \p_{\boldsymbol{\phi}}} \left [ S(\x, \z) \right ], \boldsymbol{\phi} - \boldsymbol{\phi}' \right \rangle + A(\boldsymbol{\phi}') - A(\boldsymbol{\phi}), \\
        && & = \left \langle \frac{\partial A(\boldsymbol{\phi})}{\partial \boldsymbol{\phi}}, \boldsymbol{\phi} - \boldsymbol{\phi}' \right \rangle + A(\boldsymbol{\phi}') - A(\boldsymbol{\phi}), && \triangleright \text{Since~} \frac{\partial A(\boldsymbol{\phi})}{\partial \boldsymbol{\phi}} = \expectation_{\x, \z \sim \p_{\boldsymbol{\phi}}} \left [ S(\x, \z) \right ] \\
        && & = D_{A} (\boldsymbol{\phi}', \boldsymbol{\phi}).
    \end{flalign*}
    Plug the result back to \cref{eq:proof1-2}, we conclude that $Q_{\params} (\params')$ and the following are equivalent up to a constant independent of $\params'$:
        \begin{align}
            \frac{1}{\abs{\data}} \sum_{\x \in \data} \left \langle \frac{\partial \log \p_{\params} (\x)}{\partial \params}, \params' \right \rangle - \mathtt{KL}_{\params} (\params').
            \label{eq:proof1-3}
        \end{align}
\end{proof}

We proceed to prove \cref{lem:key-terms}, which offers a practical way to compute the two key quantities in \cref{eq:proof1-3}.

\begin{proof}[Proof of \cref{lem:key-terms}]
    Recall from our definition that $\p_{\params} (\x) := \tilde{\p}_{\params} (\x) / Z(\params)$. We start by proving a key result: for each parameter $\phi_{n,c}$, the partition function
        \begin{align}
            Z (\params) = \mathtt{TD} (n) \cdot \exp (\phi_{n,c}) + C,
            \label{eq:proof2-1}
        \end{align}
    \noindent where $C$ is independent of $\phi_{n,c}$. Note that by definition $Z (\params)$ is computed by the same feedforward pass shown in \cref{eq:pc-fwd}, with the only difference that the partition function is set to $1$ for input nodes. Specifically, denote $Z_{n} (\params)$ as the partition function of node $n$, we have
        \begin{align*}
            Z_{n} (\params) = \begin{cases}
                1 & n \text{~is~an~input~node}, \\
                \prod_{c \in \ch(n)} Z_{c} (\params) & n \text{~is~a~product~node}, \\
                \sum_{c \in \ch(n)} \exp (\phi_{n,c}) \cdot Z_{c} (\params) & n \text{~is~a~sum~node}.
            \end{cases}
        \end{align*}
    Define $\mathtt{TD}_{m} (n)$ as the TD-prob of node $n$ for the PC rooted at $m$ (assume $n$ is a descendant node of $m$). We prove \cref{eq:proof2-1} by induction over $m$ in $Z_{m} (\phi)$.

    In the base case where $m = n$, we have that
        \begin{flalign*}
            && Z_{m} (\params) & = Z_{n} (\params) = \sum_{c' \in \ch(n)} \exp (\phi_{n,c'}) \cdot Z_{c'} (\params), \\
            && & = \sum_{c' \in \ch(n)} \exp (\phi_{n,c'}), && \triangleright \text{Since~we~assume~} \forall c, Z_{c} (\params) = 1 \\
            && & = \exp (\phi_{n,c}) + \sum_{c' \in \ch(n), c' \neq c} \exp (\phi_{n,c'}), \\
            && & = \mathtt{TD}_{m} (n) \cdot \exp (\phi_{n,c}) + \sum_{c' \in \ch(n), c' \neq c} \exp (\phi_{n,c'}), && \triangleright \text{Since~} \forall c, \mathtt{TD}_{c} (c) = 1 \\
            && & = \mathtt{TD}_{m} (n) \cdot \exp (\phi_{n,c}) + C.
        \end{flalign*}
    Next, assume $m$ is a sum node and \cref{eq:proof2-1} holds for all its children. That is,
        \begin{align*}
            \forall b \in \ch(m), \quad Z_{b} (\params) = \mathtt{TD}_{b} (n) \cdot \exp (\phi_{n,c}) + C.
        \end{align*}
    We proceed by plugging in the definition of $Z_{m} (\params)$:
        \begin{align}
            Z_{m} (\params) & = \sum_{b \in \ch(m)} \exp (\phi_{m,b}) \cdot Z_{b} (\params), \nonumber \\
            & = \sum_{b \in \ch(m)} \exp (\phi_{m,b}) \cdot \mathtt{TD}_{b} (n) \cdot \exp (\phi_{n,c}) + C.
            \label{eq:proof2-2}
        \end{align}
    Denote $\calA \subseteq \ch (m)$ as the set of child nodes that are ancestors of $n$, and $\calB = \ch(m) \backslash \calA$ is its complement. From the definition of TD-probs, we have
        \begin{flalign*}
            && \mathtt{TD}_{m} (n) & = \sum_{b \in \calA} \mathtt{TD}_{m} (b) \cdot \mathtt{TD}_{b} (n), \\ 
            && & = \sum_{b \in \calA} \exp (\phi_{m,b}) \cdot \mathtt{TD}_{b} (n). && \triangleright \text{Since~} \mathtt{TD}_{m} (b) = \exp (\phi_{m, b})
        \end{flalign*}
    Plug in \cref{eq:proof2-2}, we conclude that
        \begin{flalign*}
            && Z_{m} (\params) & = \sum_{b \in \calA} \exp (\phi_{m,b}) \cdot \mathtt{TD}_{b} (n) \cdot \exp (\phi_{n,c}) + \sum_{b \in \calB} \exp (\phi_{m,b}) \cdot \mathtt{TD}_{b} (n) \cdot \exp (\phi_{n,c}) + C, \\
            && & = \mathtt{TD}_{m} (n) \cdot \exp (\phi_{n,c}) + \sum_{b \in \calB} \exp (\phi_{m,b}) \cdot \mathtt{TD}_{b} (n) \cdot \exp (\phi_{n,c}) + C, && \\
            && & = \mathtt{TD}_{m} (n) \cdot \exp (\phi_{n,c}) + C'.
        \end{flalign*}
    Finally, if $m$ is a product node such that \cref{eq:proof2-1} holds for all its children, we have that 
        \begin{align*}
            Z_{m} (\params) & = \prod_{b \in \ch(m)} Z_{b} (\params).
        \end{align*}
    Since $m$ is decomposable (cf. Def.~\ref{defn:sm-dec}), there is at most one $b \in \ch(m)$ that is an ancestor of $n$ (otherwise multiple child nodes contain the variable scope of $n$). Denote that child node as $\hat{b}$, we further simplify the above equation to 
        \begin{align}
            Z_{m} (\params) = Z_{\hat{b}} (\params) = \mathtt{TD}_{\hat{b}} (n) \cdot \exp (\phi_{n,c}) + C
            \label{eq:proof2-3}
        \end{align}
    since all other terms are independent of $\phi_{n,c}$ and are assumed to be $1$. According to the definition of TD-probs, we have
        \begin{align}
            \forall b \in \ch(m), \quad \mathtt{TD}_{b} (n) = \mathtt{TD}_{m} (n).
            \label{eq:proof2-3-1}
        \end{align}
    Plug this into \cref{eq:proof2-3} gives the desired result:
        \begin{align*}
            Z_{m} (\params) = \mathtt{TD}_{m} (n) \cdot \exp (\phi_{n,c}) + C.
        \end{align*}
    This completes the proof of \cref{eq:proof1-1}.

    We continue on proving the first equality in \cref{lem:key-terms}:
        \begin{align*}
            \frac{\partial \log \p_{\params} (\x)}{\partial \params} & = \frac{\partial \log \hat{\p}_{\params} (\x)}{\partial \params} - \frac{\partial \log Z (\params)}{\partial \params}, \\
            & = \frac{\partial \log \hat{\p}_{\params} (\x)}{\partial \params} - \frac{1}{Z (\params)} \cdot \frac{\partial Z (\params)}{\partial \params}, \\
            & = \frac{\partial \log \hat{\p}_{\params} (\x)}{\partial \params} - \frac{\partial Z (\params)}{\partial \params}.
        \end{align*}
    According to \cref{eq:proof2-1}, we can simplify the derivative of $Z (\params)$ with respect to $\phi_{n,c}$ as $\mathtt{TD} (n) \cdot \exp (\phi_{n,c}) = \mathtt{TD} (\phi_{n,c})$, where the last equality follows from \cref{defn:td-probs}. Therefore, we conclude that
        \begin{align*}
            \frac{\partial \log \p_{\params} (\x)}{\partial \params} = \frac{\partial \log \hat{\p}_{\params} (\x)}{\partial \params} - \mathtt{TD} (\params).
        \end{align*}
    We move on to the second equality in \cref{lem:key-terms}. According to \citet{vergari2021compositional}, $\mathtt{KL}_{\params} (\params')$ can be computed recursively as follows (define $\mathtt{KL}_{\params}^{n} (\params')$ as the KLD \wrt $n$): 
        \begin{align}
            \mathtt{KL}_{\params}^{n} (\params') = \begin{cases}
                0 & n \text{~is~an~input~node}, \\
                \sum_{c \in \ch(n)} \mathtt{KL}_{\params}^{c} (\params') & n \text{~is~a~product~node}, \\
                \sum_{c \in \ch(n)} \exp(\phi_{n,c}) \big ( \phi_{n,c} - \phi'_{n,c} \big ) + \exp(\phi_{n,c}) \cdot \mathtt{KL}_{\params}^{c} (\params') & n \text{~is~a~sum~node}.
            \end{cases}
            \label{eq:proof2-4}
        \end{align}
    We want to show that for each $m$ that is an ancestor of $n$, the following holds:
        \begin{align}
            \mathtt{KL}_{\params}^{m} (\params') = - \mathtt{TD}_{m} (n) \cdot \exp (\phi_{n,c}) \cdot \phi'_{n,c} + C,
            \label{eq:proof2-5}
        \end{align}
    \noindent where $C$ is independent of $\phi'_{n,c}$. We can use the exact same induction procedure that is used to prove \cref{eq:proof2-1}. For all ancestor sum nodes $m$ of $n$, the first term in \cref{eq:proof2-4} (the last row among the three cases) is always independent of $\phi'_{n,c}$, and hence the recursive definition resembles that of $Z_{m} (\params)$. Specifically, for all ancestor nodes of $n$, \cref{eq:proof2-5} simplifies to 
        \begin{align*}
            \mathtt{KL}_{\params}^{n} (\params') = \begin{cases}
                0 & n \text{~is~an~input~node}, \\
                \sum_{c \in \ch (n)} \mathtt{KL}_{\params}^{c} (\params') & n \text{~is~a~product~node}, \\
                \sum_{c \in \ch(n)} \exp (\phi_{n,c}) \cdot \mathtt{KL}_{\params}^{c} (\params') & n \text{~is~a~sum~node}.
            \end{cases}
        \end{align*}
    The key difference with $Z_{n} (\params)$ is the definition of product nodes. Therefore, following the same induction proof of \cref{eq:proof2-1}, we only need to re-derive the case where $m$ is a product node such that \cref{eq:proof2-5} holds for all its children. 
    
    Since the PC is decomposable, there is only one child node $b \in \ch(m)$ that is an ancestor of $n$. Therefore, $\forall c \in \ch(m), c \neq b$, $\mathtt{KL}_{\params}^{c} (\params')$ is independent of $\phi' (n,c)$. Hence, we have
        \begin{flalign*}
            && \mathtt{KL}_{\params}^{m} (\params') & = - \mathtt{TD}_{b} (n) \cdot \exp (\phi_{n,c}) \cdot \phi'_{n,c} + C, \\
            && & = - \mathtt{TD}_{m} (n) \cdot \exp (\phi_{n,c}) \cdot \phi'_{n,c} + C. && \triangleright \text{According~to~Eq.~(\ref{eq:proof2-3-1})}
        \end{flalign*}

    Writing \cref{eq:proof2-5} in a vectorized form for every $\phi'_{n,c}$ leads to our final result:
        \begin{align*}
            \mathtt{KL}_{\params} (\params') = - \langle \mathtt{TD} (\params), \params' \rangle + C.
        \end{align*}
\end{proof}

\subsection{Derivation of the Full-Batch and Mini-Batch EM}
\label{appx:proof-constraint-em}

\paragraph{Full-Batch EM.}
Define $\calS$ as the set of all sum nodes in the PC, the constrained optimization problem is 
    \begin{gather*}
        \maximize_{\params'} \left \langle \frac{1}{\abs{\data}} \sum_{\x \in \data} \frac{\partial \log \tilde{\p}_{\params} (\x)}{\partial \params}, \params' \right \rangle, \\
        \text{s.t.~} \forall n \in \calS, \; \sum_{c \in \ch(n)} \exp (\phi'_{n,c}) = 1.
    \end{gather*}
To incorporate the constraints, we use the method of Lagrange multipliers. The Lagrangian for this problem is
    \begin{align*}
        \calL (\params', \{\lambda_{n}\}_{n \in \calS}) = \left \langle \frac{1}{\abs{\data}} \sum_{\x \in \data} \frac{\partial \log \tilde{\p}_{\params} (\x)}{\partial \params}, \params' \right \rangle + \sum_{n \in \calS} \lambda_{n} \cdot \left ( 1 - \sum_{c \in \ch(n)} \exp (\phi'_{n,c}) \right ),
    \end{align*}
\noindent where the Lagrange multipliers $\{\lambda_{n}\}_{n \in \calS}$ enforce the constraints.

To minimize the Lagrangian \wrt $\params'$, we take the partial derivative of $\calL (\params', \{\lambda_{n}\}_{n \in \calS})$ \wrt each $\phi'_{n,c}$ and set it to 0:
    \begin{align*}
        \frac{\partial \calL (\params', \{\lambda_{n}\}_{n \in \calS})}{\partial \phi'_{n,c}} = \mathtt{F}^{\data}_{\params} (n, c) - \lambda_{n} \exp (\phi'_{n,c}) = 0,
    \end{align*}
\noindent where $\mathtt{F}^{\data}_{\params} (n, c)$ is defined in \cref{sec:full-batch-em}. Simplifying this equation gives:
    \begin{align*}
        \phi'_{n,c} = \log \mathtt{F}^{\data}_{\params} (n, c) - \log Z,
    \end{align*}
\noindent where $Z = \sum_{c' \in \ch(n)} \mathtt{F}_{\params}^{\data} (n, c')$.

\paragraph{Mini-Batch EM.}
Similar to the full-batch case, according to \cref{sec:mini-batch-em}, the constrained optimization problem is
    \begin{gather*}
        \maximize_{\params'} \left \langle \frac{1}{\abs{\data}} \sum_{\x \in \data} \frac{\partial \log \tilde{\p}_{\params} (\x)}{\partial \params} + (\gamma - 1) \cdot \mathtt{TD} (\params), \params' \right \rangle, \\
        \text{s.t.~} \forall n \in \calS, \; \sum_{c \in \ch(n)} \exp (\phi'_{n,c}) = 1.
    \end{gather*}
Following the full-batch case, the Lagrangian is given by
    \begin{align*}
        \calL (\params', \{\lambda_{n}\}_{n \in \calS}) = \left \langle \frac{1}{\abs{\data}} \sum_{\x \in \data} \frac{\partial \log \tilde{\p}_{\params} (\x)}{\partial \params} + (\gamma - 1) \cdot \mathtt{TD} (\params), \params' \right \rangle + \sum_{n \in \calS} \lambda_{n} \cdot \left ( 1 - \sum_{c \in \ch(n)} \exp (\phi'_{n,c}) \right ).
    \end{align*}
To minimize the Lagrangian with respect to $\params'$, we compute the partial derivative of $\calL (\params', \{\lambda_{n}\}_{n \in \calS})$ \wrt each $\phi'_{n,c}$ and set it equal to zero:
    \begin{align*}
        \frac{\partial \calL (\params', \{\lambda_{n}\}_{n \in \calS})}{\partial \phi'_{n,c}} = \mathtt{F}^{\data}_{\params} (n, c) + (\gamma - 1) \cdot \mathtt{TD} (\phi'_{n,c}) - \lambda_{n} \exp (\phi'_{n,c}) = 0.
    \end{align*}
Using the definition $\mathtt{TD} (\phi'_{n,c}) = \mathtt{TD}_{\params} (n) \cdot \exp (\phi_{n,c})$, the solution is given by
    \begin{align*}
        \phi'_{n,c} = \log \Big ( \mathtt{TD}_{\params} (n) \cdot \exp (\phi_{n,c}) + \alpha \cdot \mathtt{F}^{\data}_{\params} (n, c) \Big ) - \log Z,
    \end{align*}
\noindent where $\alpha := 1 / (\gamma - 1)$ and $Z = \sum_{c \in \ch(n)} \mathtt{TD}_{\params} (n) \cdot \exp (\phi_{n,c}) + \alpha \cdot \mathtt{F}^{\data}_{\params} (n, c)$.

\subsection{Decomposition of Parameter Flows}
\label{appx:proof-normed-flows}

In this section, we show that $\sum_{c \in \ch(n)} \hat{\mathtt{F}}_{\params}^{\x} (n, c) = 1$, where $n$ is a sum node. We start from the definition of $\hat{\mathtt{F}}_{\params}^{\x} (n, c)$:
    \begin{flalign*}
        && \hat{\mathtt{F}}_{\params}^{\x} (n, c) & = \frac{\partial \log \tilde{\p}_{\params}^{n} (\x)}{\partial \phi_{n,c}} = \frac{1}{\tilde{\p}_{\params}^{n} (\x)} \cdot \frac{\partial \tilde{\p}_{\params}^{n} (\x)}{\partial \phi_{n,c}}, \\
        && & = \frac{\theta_{n,c}}{\tilde{\p}_{\params}^{n} (\x)} \cdot \frac{\partial \tilde{\p}_{\params}^{n} (\x)}{\partial \theta_{n,c}}, && \triangleright \text{By~definition~} \theta_{n,c} = \exp (\phi_{n,c}) \\
        && & = \frac{\theta_{n,c} \cdot \tilde{\p}_{\params}^{c} (\x)}{\tilde{\p}_{\params}^{n} (\x)}.
    \end{flalign*}
Now we have
    \begin{align*}
        \sum_{c \in \ch(n)} \hat{\mathtt{F}}_{\params}^{\x} (n, c) = \sum_{c \in \ch(n)} \frac{\theta_{n,c} \cdot \tilde{\p}_{\params}^{c} (\x)}{\tilde{\p}_{\params}^{n} (\x)} = 1.
    \end{align*}

\subsection{Computing TD-Prob Using Auto-Differentiation}
\label{appx:proof-tdp-ad}

In this section, we prove that $\mathtt{TD} (n)$ (and thus also $\mathtt{TD} (\phi_{n,c})$) can be computed by differentiating $Z_{n_r} (\params)$, where $n_r$ is the root node. Note that the definition of $Z_{n} (\params)$ follows \cref{appx:proof-em-general-form}.

We proceed with the proof by induction. First, as a base case, we have that
    \begin{align*}
        \frac{\partial Z_{n_r} (\params)}{\partial Z_{n_r} (\params)} = 1 = \mathtt{TD} (n_r).
    \end{align*}
Next, assume that for a sum/input node $n$, for all its parents $m \in \pa (n)$ (which are product nodes according to Def.~\ref{defn:pc}) we satisfy that
    \begin{align*}
        \mathtt{TD} (m) = \frac{\partial Z_{n_r} (\params)}{\partial Z_{m} (\params)}.
    \end{align*}
We proceed to derive $\partial Z_{n_r} (\params) / \partial Z_{n} (\params)$:
    \begin{flalign*}
        && \frac{\partial Z_{n_r} (\params)}{\partial Z_{n} (\params)} & = \sum_{m \in \pa(n)} \frac{\partial Z_{n_r} (\params)}{\partial Z_{m} (\params)} \cdot \frac{\partial Z_{m} (\params)}{\partial Z_{n} (\params)}, \\
        && & = \sum_{m \in \pa(n)} \mathtt{TD} (m) \cdot \frac{\partial Z_{m} (\params)}{\partial Z_{n} (\params)}, \\
        && & = \sum_{m \in \pa(n)} \mathtt{TD} (m). && \triangleright \text{By~definition~} \frac{\partial Z_{m} (\params)}{\partial Z_{n} (\params)} = 1
    \end{flalign*}
The final case is for a product node $n$, assuming all its parents satisfy the requirement, \ie $\forall m \in \pa (n), \mathtt{TD} (m) = \partial Z_{n_r} (\params) / \partial Z_{m} (\params)$. We can simplify the gradient with respect to $Z_{n} (\params)$ by
    \begin{flalign*}
        && \frac{\partial Z_{n_r} (\params)}{\partial Z_{n} (\params)} & = \sum_{m \in \pa(n)} \frac{\partial Z_{n_r} (\params)}{\partial Z_{m} (\params)} \cdot \frac{\partial Z_{m} (\params)}{\partial Z_{n} (\params)}, \\
        && & = \sum_{m \in \pa(n)} \mathtt{TD} (m) \cdot \frac{\partial Z_{m} (\params)}{\partial Z_{n} (\params)}, \\
        && & = \sum_{m \in \pa(n)} \mathtt{TD} (m) \cdot \theta_{m,n}. && \triangleright \text{By~definition~} \frac{\partial Z_{m} (\params)}{\partial Z_{n} (\params)} = \theta_{m,n}
    \end{flalign*}

\section{Global Parameter Renormalization of PCs}
\label{appx:global-renorm}

In this section, we propose a simple renormalization algorithm that takes an unnormalized PC $\p_{\params} (\X)$ (\ie its partition function does not equal $1$) with parameters $\params$ and returns a new set of parameters $\params'$ such that for each node $n$ in the PC
    \begin{align*}
        \forall \x, \; \tilde{\p}_{\params'}^{n} (\x) = \frac{1}{Z_{n} (\params)} \cdot \tilde{\p}_{\params}^{n} (\x),
    \end{align*}
\noindent where $Z (\params) := \sum_{\x} \tilde{\p}_{\params}^{n} (\x)$ is the partition function of $\tilde{\p}_{\params}^{n}$. 

% \paragraph{The Algorithm.}
% First, we perform a feedforward evaluation of the PC to compute the partition function $Z_{n} (\params)$ of every node $n$:
%     \begin{align*}
%         Z_{n} (\params) = \begin{cases}
%             1 & n \text{~is~an~input~node}, \\
%             \prod_{c \in \ch(n)} Z_{c} (\params) & n \text{~is~a~product~node}, \\
%             \sum_{c \in \ch(n)} \exp (\phi_{n,c}) \cdot Z_{c} (\params) & n \text{~is~a~sum~node}.
%         \end{cases}
%     \end{align*}
% Next, for every sum edge $(n, c)$ (\ie $n$ is a sum node and $c$ is one of its children), we update the parameter as
%     \begin{align}
%         \phi'_{n,c} = \log \left ( \frac{\theta_{n,c} \cdot Z_{\params} (c)}{Z_{\params} (n)} \right ),
%         \label{eq:renorm}
%     \end{align}
% \noindent where $\theta_{n,c} := \exp (\phi_{n,c})$. The existence of this normalization algorithm has been previously shown by \citet{martens2014expressive}, although they did not provide an easy-to-implement algorithm.

\paragraph{Analysis.}
We begin by proving the correctness of the algorithm. Specifically, we show by induction that $\tilde{\p}_{\params'}^{n} (\x) = \tilde{\p}_{\params}^{n} (\x) / Z_{\params} (n)$ for every $n$ and $\x$. In the base case, all input nodes satisfy the equation since they are assumed to be normalized. Next, given a product node $n$, assume the distributions encoded by all its children $c$ satisfy 
    \begin{align}
        \forall c \in \ch(n), \; \tilde{\p}_{\params'}^{c} (\x) = \tilde{\p}_{\params}^{c} (\x) / Z_{\params} (c).
        \label{eq:proof3-1}
    \end{align}
Then by definition, $\p_{\params'}^{n} (\x)$ can be written as:
    \begin{align*}
        \tilde{\p}_{\params'}^{n} (\x) & = \prod_{c \in \ch(n)} \tilde{\p}_{\params'}^{c} (\x) = \prod_{c \in \ch(n)} \tilde{\p}_{\params}^{c} (\x) / Z_{\params} (c), \\
        & = \frac{\prod_{c \in \ch(n)} \tilde{\p}_{\params}^{c} (\x)}{\prod_{c \in \ch(n)} Z_{\params} (c)}, \\
        & = \frac{\tilde{\p}_{\params}^{n} (\x)}{Z_{\params} (n)}.
    \end{align*}
Finally, consider a sum node $n$ whose children satisfy \cref{eq:proof3-1}. We simplify $\tilde{\p}_{\params'}^{n} (\x)$ in the following:
    \begin{flalign*}
        && \tilde{\p}_{\params'}^{n} (\x) & = \sum_{c \in \ch(n)} \theta'_{n,c} \cdot \tilde{\p}_{\params'}^{c} (\x), \\
        && & = \sum_{c \in \ch(n)} \frac{\theta_{n,c} \cdot Z_{\params} (c)}{Z_{\params} (n)} \cdot \tilde{\p}_{\params'}^{c} (\x), && \triangleright \text{According~to~Eq.~(\ref{eq:renorm})} \\
        && & = \sum_{c \in \ch(n)} \frac{\theta_{n,c} \cdot \cancel{Z_{\params} (c)}}{Z_{\params} (n)} \cdot \frac{\tilde{\p}_{\params}^{c} (\x)}{\cancel{Z_{\params} (c)}}, && \triangleright \text{By~induction} \\
        && & = \frac{\sum_{c \in \ch(n)} \theta_{n,c} \cdot \tilde{\p}_{\params}^{c} (\x)}{Z_{\params} (n)}, \\
        && & = \tilde{\p}_{\params}^{n} (\x) / Z_{\params} (n).
    \end{flalign*}

    We proceed to show an interesting property of the proposed global renormalization.

    \begin{lem}
        Given a PC $\p_{\params} (\X)$. Denote $\params'$ as the parameters returned by the global renormalization algorithm. Then, for every sum edge $(n, c)$, we have
            \begin{align*}
                \forall \x, \; \frac{\partial \log \tilde{\p}_{\params'} (\x)}{\partial \phi'_{n,c}} = \frac{\partial \log \tilde{\p}_{\params} (\x)}{\partial \phi_{n,c}}.
            \end{align*}
    \end{lem}

    \begin{proof}
        We begin by showing that for each node $n$ is one of its children, the following holds:
            \begin{align*}
                \forall \x, \; \frac{\partial \log \tilde{\p}_{\params'}^{n} (\x)}{\partial \log \tilde{\p}_{\params'}^{c} (\x)} = \frac{\partial \log \tilde{\p}_{\params}^{n} (\x)}{\partial \log \tilde{\p}_{\params}^{c} (\x)}.
            \end{align*}
        If $n$ is a product node, both the left-hand side and the right-hand side equal $1$. For example, consider the left-hand side. According to the definition, we have 
            \begin{align*}
                \log \tilde{\p}_{\params'}^{n} (\x) = \sum_{c \in \ch(n)} \log \tilde{\p}_{\params'}^{c} (\x).
            \end{align*}
        Hence, its derivative \wrt $\log \tilde{\p}_{\params'}^{c} (\x)$ is $1$.

        If $n$ is a sum node, then for each $\x$, we have
            \begin{align*}
                \frac{\partial \log \tilde{\p}_{\params'}^{n} (\x)}{\partial \log \tilde{\p}_{\params'}^{c} (\x)} & = \frac{\tilde{\p}_{\params'}^{c} (\x)}{\tilde{\p}_{\params'}^{n} (\x)} \cdot \frac{\partial \tilde{\p}_{\params'}^{n} (\x)}{\partial \tilde{\p}_{\params'}^{c} (\x)}, \\
                & = \frac{\tilde{\p}_{\params'}^{c} (\x)}{\tilde{\p}_{\params'}^{n} (\x)} \cdot \theta'_{n,c}, \\
                & = \frac{\tilde{\p}_{\params}^{c} (\x) / Z_{\params} (c)}{\tilde{\p}_{\params}^{n} (\x) / Z_{\params} (n)} \cdot \theta'_{n,c}, \\
                & = \frac{\tilde{\p}_{\params}^{c} (\x) / \cancel{Z_{\params} (c)}}{\tilde{\p}_{\params}^{n} (\x) / \cancel{Z_{\params} (n)}} \cdot \frac{\theta_{n,c} \cdot \cancel{Z_{\params} (c)}}{\cancel{Z_{\params} (n)}}, \\
                & = \frac{\tilde{\p}_{\params}^{c} (\x)}{\tilde{\p}_{\params}^{n} (\x)} \cdot \theta_{n,c}, \\
                & = \frac{\tilde{\p}_{\params}^{c} (\x)}{\tilde{\p}_{\params}^{n} (\x)} \cdot \frac{\partial \tilde{\p}_{\params}^{n} (\x)}{\partial \tilde{\p}_{\params}^{c} (\x)}, \\
                & = \frac{\partial \log \tilde{\p}_{\params}^{c} (\x)}{\partial \log \tilde{\p}_{\params}^{n} (\x)}.
            \end{align*}
    \end{proof}

\section{Additional Experimental Details}
\label{appx:exps}

\subsection{Details about the Datasets and the PC Models}
\label{appx:datasets}
\paragraph{ImageNet32 and The Corresponding PCs.}
For ImageNet32, we partition every $32\times32$ image (three color channels) into four $16\times16$ patches and treat these as individual data samples. There are in total $16 \times 16 \times 3 = 768$ categorical variables in the PC.

We preprocess the data in color space with a lossy transformation, YCoCg, and its scaled, reversible variant, YCoCg-R, proposed by \citet{malvar2003ycocg}. Specifically, in YCoCg transformation, given a pixel with RGB values $(R, G, B)$, we first normalize them to the range $[0, 1]$ by
    \begin{align*}
        r = R / 255, \; g = G / 255, \; b = B / 255.
    \end{align*}
We then apply the following linear transformation:
    \begin{align*}
        co  = r - b, \; tmp = b + co/2, \; cg  = g - tmp, \; y   = tmp*2 + cg + 1,
    \end{align*}
\noindent where $y$, $co$, and $cg$ are all in the range $[-1, 1]$. Finally, we quantize the interval $[-1, 1]$ into $256$ bins uniformly and convert $y$, $co$, and $cg$ to their quantized version $Y$, $Co$, and $Cg$, respectively. Note that $Y$, $Co$, and $Cg$ are all categorical variables with $256$ categories.

And the other transformation, YCoCg-R, maps 8-bit integer RGB values to YCoCg values without information loss. Similarly, the forward transformation is given by:
\begin{align*} 
co = r - b, \;  tmp = b + co/2, \; cg  = g - tmp, \; y = tmp + cg / 2,
\end{align*}
\noindent where $y$, $co$, and $cg$ are also in the range $[-1, 1]$. The resulting integer channels are treated as categorical variables with $512$ categories.

We train several deep PC architectures on the ImageNet32 dataset \cite{deng2009imagenet}. These include Hidden Chow-Liu Trees (HCLT) \citep{liu2021tractable} with hidden size 512 and 1024, and Partitioned Data HCLTs (PDHCLTs) with 256 and 512 latents.\footnote{The PDHCLT structure is described in \cref{appx:pdhclt}.} The PDHCLT models are configured to partition the input data, which has a shape of (3, 16, 16). They use a maximum of 8 connections between product blocks. Please refer to 

\paragraph{WikiText and The Corresponding PCs.}
We also extend our empirical evaluation to language modeling using the WikiText-103 dataset by \citet{merity2017pointer}. The raw text data is preprocessed into a format suitable for sequence modeling. Specifically, we firstly tokenize the entire corpus using the standard GPT-2 tokenizer by \cite{GPT2}. All tokenized documents are then concatenated into a single continuous stream of tokens. Finally, this stream is partitioned into non-overlapping sequences of a fixed length of 128 tokens, with any remaining tokens at the end discarded to ensure uniformity across samples. On this preprocessed data, we trained Hidden Markov Models (HMMs), with 256, 512, and 1024 hidden states, and Monarch HMM with a size of 1024 \citep{zhang2025scaling}.

\paragraph{Biobank dataset and The Corresponding PCs.}
For our bioinformatics experiments, we use genetic data sourced from the UK Biobank (UKBB) dataset \citep{bioBank}. This specific version focuses on a single Linkage Disequilibrium (LD) block located on chromosome 6. The authors of this dataset remove SNPs with more than 1\% missingness and those that deviate significantly from Hardy-Weinberg Equilibrium ($1 \times 10^{-7}$ significance). Further, only individuals with no genetic relatedness to other individuals are considered. We have gotten the approval from the UK BioBank to access this dataset.

\subsection{Details about the Optimizers}
\label{appx:optimizers}
\paragraph{Full-Batch EM.}
The full-batch EM implementation follows prior work (\eg \citet{choi2021group,peharz2020einsum}).

\paragraph{An\textbf{em}one.}
For notation simplicity, we define $\alpha = \eta / (\eta - 1)$. Therefore, we can rewrite \cref{eq:ours-mini-em} equivalently as
    \begin{align}
        \theta'_{n,c} = \left ( (1 - \alpha) \cdot \mathtt{TD}_{\params} (n) \cdot \theta_{n,c} + \alpha \cdot \mathtt{F}_{\params}^{\data} (n, c) \right ) / Z,
        \label{eq:new-update}
    \end{align}
\noindent which makes it more consistent with the baseline mini-batch EM algorithm. 

% For both our algorithm and the baseline algorithm in \cref{eq:old-mini-em}, we choose a base learning rate of $\alpha = 0.4$ and a cosine learning rate decay schedule to decrease it to $\alpha = 0.08$. For both algorithms, we use a batch size of $32768$. 

We performed a preliminary hyperparameter search where we experimented with batch sizes including 512 and 16384. For the learning rate, we tested fixed values of $\alpha \in \{0.1,0.2,0.4,0.6\}$ and also employed a cosine decay schedule. The schedules included decreasing the rate from a base of $\alpha=0.4$ to a final rate of $\alpha=0.2$, and from $\alpha=0.8$ down to $\alpha=0.6$.  We also set a momentum of $0.9$ in pratice. We select the final hyperparameter setting based on performance after the first 100 epochs.

\paragraph{Gradient-Based.}
Following \citet{loconte2025sum,loconte2024relationship}, we adopt the Adam optimizer \citep{kingma2014adam}. We selected hyperparameters using a similar search criterion as our EM experiments, testing learning rates of $\{1\times 10^{-2}, 3\times 10^{-3}, 5\times 10^{-3}, 1\times 10^{-3}\}$ and batch sizes of 512 and 1024. On the ImageNet32 YCC dataset, we found a learning rate of $1\times 10^{-2}$ performed the best, which aligns with the observations in \citet{loconte2024relationship}. To ensure correctness, we first validated our implementation by reproducing prior results on the MNIST dataset, achieving a log-likelihood of -661.6 after 30 epochs.

% We tested with learning rates $1e-2$, $3e-3$, and $1e-3$, and observed that $1e-2$ performed the best. This matches the observation in \citet{loconte2024relationship}. We also tested different batch sizes, and $1024$ performs the best given a fixed number of epochs.

% \begin{figure}[t]
%     \centering
%     \includegraphics[width=0.4\linewidth]{arxiv/figures/fig-exp-result.png}
%     \caption{Validation log-likelihood of an HCLT PC \citep{liu2021tractable} on $16\times16$ patches from ImageNet32. Our proposed mini-batch EM algorithms outperform all other optimizers by a large margin.}
%     \label{fig:exp-result}
% \end{figure}

% \paragraph{Empirical Insights.}
% To have a fair comparison between full-batch and mini-batch algorithms, we plot the validation log-likelihood vs. the total number of consumed samples. Results are shown in \cref{fig:exp-result}. First, we observe that EM-based algorithms perform better than the gradient-based method Adam, suggesting that EM is a better parameter learning algorithm for PCs. Next, the proposed mini-batch EM algorithm outperforms both the full-batch EM algorithm and the mini-batch EM algorithm proposed by \citet{peharz2020einsum}, indicating the empirical superiority of the proposed algorithm.

\subsection{Details about computing resources}
We ran all the experiments included on NVIDIA A40s and NVIDIA GeForce RTX 4090.

\subsection{Convergence experiments}
\label{appx:exp-convergence}

We provide a similar experiment to \cref{tab:convergence_speed}, to demonstrate the consistency of faster convergence speed on various datasets, as shown in \cref{tab:convergence_speed_wikitext}.
\begin{table}[ht]
\centering
\caption{\textbf{Convergence speed (epochs) for HMM 256 on WikiText.} The table reports epochs to reach specific LL thresholds, with $\Delta$ representing the difference from the best LL of -722 (\cref{tab:density_estimation_text}). Lower is better. Bold marks the best result per column; $\infty$ indicates failure to reach the threshold in time.}
\label{tab:convergence_speed_wikitext}
\setlength{\tabcolsep}{4pt} 
\begin{tabular}{@{}l ccc@{}}
\toprule
\textbf{Method} & \textbf{LL $\geq$ -730} & \textbf{LL $\geq$ -724} & \textbf{LL $\geq$ -723} \\
& \small($\Delta \approx 7.8$) & \small($\Delta \approx 1.8$) & \small($\Delta \approx 0.8$) \\
\midrule
Full EM & 60 & 230 & 375 \\
Adam & $\infty$ & $\infty$ & $\infty$ \\
Mini EM & 45 & $\infty$ & $\infty$ \\
An\textbf{em}one & \textbf{30} & \textbf{115} & \textbf{275} \\
\bottomrule
\end{tabular}
\end{table}

\section{The PDHCLT Structure}
\label{appx:pdhclt}

We adopt the PDHCLT structure implemented in the PyJuice \citep{liu2024scaling} package (\texttt{pyjuice.structures.PDHCLT}). In the case of images with size $(3, 16, 16)$, we define a ``split interval'' to be $(3, 4, 4)$, which means that we partition the image into chunks of size $3 \times 4 \times 4$, resulting in $4 \times 4$ chunks. For each chunk, we adopt the HCLT structure, and the PD structure is used to connect the different chunks.

For the BioBank dataset, the sequences have length $1167$, and we partition them into chunks of size $128$ (the last chunk has a smaller size). Again, HCLT is used to model intra-chunk dependencies and the PD structure is used to capture inter-chunk dependencies.

\end{document}